\documentclass[runningheads]{llncs}

\usepackage{amsmath,amssymb}
\usepackage[obeyFinal]{todonotes}
\usepackage{enumerate}
\usepackage{paralist}
\usepackage{wasysym}
\usepackage{tabularx}
\usepackage{hyperref}
\usepackage{paralist}
\usepackage{tikz}
\usepackage{paralist}
\usepackage{nicefrac}

\usetikzlibrary{arrows,shapes,snakes,automata,backgrounds,positioning}

\tikzstyle{player1}=[state,draw,rectangle,align=center]
\tikzstyle{player2}=[state,draw,rounded rectangle,align=center]
\tikzstyle{widget}=[draw=red,rectangle, rounded rectangle=10pt,dashed,minimum size=6mm,fill=yellow]
\tikzset{every loop/.style={looseness=7}}

\makeatletter
\let\epsilon\varepsilon
\makeatother

\newcommand*{\fun}[3]{#1 \colon #2 \to{} #3 }

\newcommand{\calM}{\mathcal{M}}
\newcommand{\pmin}{\pi_{\mathrm{min}}}
\newcommand{\sth}{\mathrel{\mid}}
\newcommand{\dist}{\ensuremath{p} }
\newcommand{\prob}{\ensuremath{\delta} }
\newcommand{\dists}{\ensuremath{\mathcal{D}} }
\newcommand{\supp}{\ensuremath{\mathsf{Supp}} }
\newcommand{\dom}{\ensuremath{\mathsf{Dom}} }
\newcommand{\struct}{\ensuremath{\mathsf{struct}} }

\newcommand{\fn}{\ensuremath{\mathsf{f}} }

\newcommand{\rvprob}{\mathbb{P}}
\newcommand{\rvexpect}[1]{\mathbb{E}\left[#1\right]}

\newcommand{\eclose}[2]{\mathrel{\sim^{#1}_{#2}}}
\newcommand{\nat}{\ensuremath{\mathbb{N}} }
\newcommand{\seq}{\ensuremath{\mathbb{S}} }
\newcommand{\N}{\mbox{I$\!$N}}
\newcommand{\realpos}{\mathbb{R}_{\geq 0}}
\newcommand{\zug}[1]{\langle #1  \rangle}
\newcommand{\stam}[1]{}
\newcommand{\rat}{\ensuremath{\mathbb{Q}} }
\newcommand{\C}{\mathcal{C}}
\newcommand{\A}{\mathcal{A}}
\newcommand{\D}{\mathcal{D}}
\newcommand{\playerOne}{\Box} 
\newcommand{\playerTwo}{\ocircle} 
\newcommand{\prefixes}[1]{\ensuremath{\mathsf{Prefs}(#1)} }
\newcommand{\prefixesPlayer}[3]{\ensuremath{\mathsf{Prefs}_{#2}(#3)} }
\newcommand{\prefix}{\ensuremath{\rho} }
\newcommand{\graph}{\ensuremath{\mathcal{G}} }

\newcommand{\last}[1]{\ensuremath{\mathsf{Last}(#1)} }
\newcommand{\outcomesMC}[2]{\ensuremath{\mathsf{Outs}_{#1}(#2)} }
\newcommand{\initV}{\ensuremath{v_{{\sf init}}} }

\newcommand{\badV}{\ensuremath{V_{{\sf bad}}} }
\newcommand{\mpay}{\ensuremath{\mathsf{MC}} }
\newcommand{\play}{\ensuremath{\pi} }

\newcommand{\Cmax}{\ensuremath{C_{\textrm{max}}}}
\newcommand{\Dmax}{\ensuremath{D_{\textrm{max}}}}
\newcommand{\Amax}{\ensuremath{A_{\textrm{max}}}}
\newcommand{\upd}{\ensuremath{\mathsf{norm}}}
\newcommand{\cst}{\ensuremath{\mathsf{cost}}}
\newcommand{\safe}{\ensuremath{\mathsf{safe}}}

\newcommand{\expect}{\ensuremath{\mathbb{E}} }
\newcommand{\decr}{\ensuremath{\mathsf{dec}}}
\newcommand{\src}{\ensuremath{\mathsf{src}}}
\newcommand{\trg}{\ensuremath{\mathsf{trg}}}
\newcommand{\post}{\ensuremath{\mathsf{Post}} }
\renewcommand{\L}{\mathcal{L}}

\newcommand{\plays}[1]{\ensuremath{\mathsf{Plays}(#1)} }
\newcommand{\markovChain}{\ensuremath{M} }
\newcommand{\act}{\ensuremath{\mathsf{active}}}
\newcommand{\mis}{\ensuremath{\mathsf{dlmiss}}}
\renewcommand{\to}[1]{\xrightarrow{#1}}
\renewcommand{\paragraph}[1]{\noindent\textbf{#1}}
\DeclareMathOperator*{\argmin}{arg\,min}

\begin{document}

\title{Safe Learning for Near-Optimal Scheduling\thanks{This work was supported by the ARC ``Non-Zero Sum Game Graphs'' project (F\'ed\'eration Wallonie-Bruxelles), the EOS ``Verilearn'' project (F.R.S.-FNRS \& FWO), and the FWO ``SAILor'' project (G030020N).}
}

\titlerunning{Safe Learning for Near-Optimal Scheduling}

\author{Damien Busatto-Gaston\inst{1} \and
Debraj Chakraborty\inst{1} \and
Shibashis Guha\inst{2} \and
Guillermo A. P\'erez\inst{3} \and
Jean-François Raskin\inst{1}
}

\authorrunning{D. Busatto-Gaston et al.}

\institute{Universit\'{e} libre de Bruxelles, Belgium \and
Tata Institute of Fundamental Research, India \and
University of Antwerp -- Flanders Make, Belgium
}

\maketitle

\begin{abstract}
  In this paper, we investigate the combination of synthesis, model-based learning, and online sampling techniques to obtain safe and near-optimal schedulers for a preemptible task scheduling problem. Our algorithms can handle Markov decision processes (MDPs) that have $10^{20}$ states and beyond which cannot be handled with state-of-the art probabilistic model-checkers. We provide probably approximately correct (PAC) guarantees for learning the model. Additionally, we extend Monte-Carlo tree search with advice, computed using safety games or obtained using the earliest-deadline-first scheduler, to safely explore the learned model online. Finally, we implemented and compared our algorithms empirically against shielded deep $Q$-learning on large task systems.
\end{abstract}


\section{Introduction}

In this paper, we show how to combine synthesis, model-based learning, and online sampling % and learning
techniques %(model-based and model-free)
to solve a scheduling problem
featuring both hard and soft constraints. We investigate solutions to
this problem both from a theoretical and from a more pragmatic point
of view. On the theoretical side, we show how safety guarantees (as
understood in formal verification) can be combined with guarantees
offered by the probably approximately correct (PAC) learning
framework~\cite{Valiant84}. On the pragmatic side,
we show how safety guarantees obtained from automatic synthesis can be
combined with Monte-Carlo tree search (MCTS)~\cite{DBLP:journals/nature/SilverHMGSDSAPL16}
%based on deep %$Q$-learning~\cite{deeprl,Mnih2015}
to offer a scalable and practical solution to solve the scheduling problem at hand. %We further compare our algorithms against deep $Q$-learning algorithms 

The scheduling problem that we consider %has been introduced in~\cite{ggr18} and
is defined as follows. A task system is composed of a set of $n$ preemptible tasks $(\tau_i)_{i \in [n]}$ %$((\tau_i)_{i \in [n]}, F, H)$ 
partitioned into a set $F$ of soft tasks and a set  $H$ of hard tasks. Time is assumed to be discrete and measured e.g. in CPU ticks.  Each task $\tau_i$ generates an infinite number of instances $\tau_{i,j}$, called \emph{jobs}, with $j=1,2, \dots$ Jobs generated by both hard and soft tasks are equipped with deadlines, which are relative to the respective arrival times of the jobs in the system. The computation time requirements of the jobs follow a discrete probability distribution, and are unknown to the scheduler but upper bounded by their relative deadline. Jobs generated by hard tasks must complete before their respective deadlines. For jobs generated by soft tasks, deadline misses result in a penalty/cost. The tasks are assumed to be independent and generated stochastically: the occurrence of a new job of one task does not depend on the occurrences of jobs of other tasks, and both the inter-arrival and computation times of jobs are independent random variables.
% whose behaviors are defined by probability distributions. 
The scheduling problem consists in finding a \emph{scheduler}, i.e.~a function that associates, to all CPU ticks, 
%the identity of the 
a task that must run at that moment; in order to:
\begin{inparaenum}[(i)]
\item avoid deadline misses by hard tasks; and
\item minimise the mean cost of deadline misses by soft tasks.  
\end{inparaenum}

% The MDP model proposed in~\cite{ggr18} is rich enough 
In~\cite{ggr18}, we modelled the
%The 
semantics of the task system
%has been modelled 
using a Markov decision process (MDP) 
%in \cite{ggr18},
and posed the problem of computing
%the task there is
%to compute
an optimal and safe scheduler. However,
that work
%it 
assumes that the distribution of all tasks is known a priori which may be unrealistic.
%In the current paper
Here, we investigate learning techniques to build algorithms that can schedule safely and optimally a set of hard and soft tasks if only the 
% structure 
deadlines and the domains of the distributions describing the tasks
of the system are known a priori and not the exact distributions. This is a more realistic assumption. Our motivation was also to investigate the joint application of both synthesis techniques coming from the field of formal verification and learning techniques on %a clean and
an understandable yet challenging setting. 
%It is easy to generate instances of the scheduling problem of different sizes in order to test the scalability of the new algorithms.
%From the scheduling problem it is easy to generate examples of various size on which the scalability of new algorithms can be tested\sgcomment{This sentence is not clear; seems to lose the context}. 

%We detail our contributions in the next paragraph and then give detailed comparison with related works.

\paragraph{{\bf Contributions.}}
First, we show the
distributions underlying a task system with only soft tasks are
\emph{efficiently} PAC learnable: by executing the task system for a
polynomial number of steps, enough samples can be collected to infer
$\epsilon$-accurate approximations of the distributions with high
probability (Thm.~\ref{thm:onlysoft}). 

Then, we consider the
general case of systems with both hard and soft tasks. Here, safe PAC learning is {\em not} always possible, and we
identify 
%\track{characterize}
two algorithmically-checkable sufficient conditions for task
systems to be safely learnable (Thms.~
\ref{thm:good_sampling} and~\ref{thm:good_efficient_sampling}).
%The first condition identifies
%task systems for which safe PAC learning exists
%(Theorem~\ref{thm:good_sampling}). To ensure {\em efficient} safe PAC
%learning, we %have identified
%propose a second stronger condition
%(Theorem~\ref{thm:good_efficient_sampling}).
These crucially depend on the underlying MDP being a single maximal end-component, as is the case in our setting (Lem. \ref{lem:singleMEC}).
%\track{These characterizations depend on a crucial property that the underlying MDP has a single maximal end-component.
%This is indeed the case for the task system that we consider here, and has been proved in Lemma \ref{lem:singleMEC}}.
Subsequently, we can use {\em
  robustness} results on MDPs to compute or learn
  %show that 
  near-optimal safe
strategies %can be computed 
from the learnt models
(Thm.~\ref{thm:robust-opt}).
%
%
% Second, we consider model-free learning: the learning procedure does not aim at discovering a model of the task system (its underlying distributions) but tries to directly learn a strategy to schedule jobs efficiently. We base our approach on deep $Q$-learning, a version of $Q$-learning that uses neural networks to succinctly represent an approximation of the learnt strategy. Such compact representations are needed as the underlying state space of a task system is exponential in its number of tasks and so it quickly becomes unmanageable for $Q$-learning. 
%5Unfortunately, 
% deep $Q$-learning alone cannot be a solution to our scheduling problem: 
%Q$-learning does not guarantee safety in scheduling the hard tasks. To overcome this problem, we combine synthesis of the most general safe scheduler (that does not rely on the distributions but only on the structure of the task system) 
%For the learning part, we apply shielded $Q$-learning in the sense of~\cite{shields}. %We then define a shielded deep $Q$-learning algorithm that safely solves our scheduling problem.

Third, in order to evaluate the relevance of our algorithms, we present experiments of a prototype implementation. These empirically validate the efficient PAC guarantees. Unfortunately, the learnt models are often too large for the probabilistic model-checking tools. In contrast, the MCTS-based algorithm scales to larger examples: e.g.
%. As an example, we
%were able to 
we learn safe scheduling strategies for systems with more than $10^{20}$ states.
%Finally, we also show that without the shielding, model-free techniques alone cannot be used to ensure safety for the hard tasks. 
%In our experimental setup, we also studied a learning methodology where very high costs are incurred when a hard task misses its deadline.
%Our experimental results are in accordance with the principal idea of Q-learning that it only provides guarantees about the expectation of the accumulated rewards, and not on safety.
Our experiments also show that a strategy obtained using deep $Q$-learning~\cite{deeprl,Mnih2015} by assigning high costs to missing deadlines of hard tasks does not respect safety, even if one learns for a long period of time and the deadline-miss costs of hard tasks are very high (cf.~\cite{shields}).
%This motivates the use of a shield during the deployment of a learnt scheduler.}
% Finally, we believe that this scheduling problem is an ideal playground for testing new combinations of verification and learning algorithms and we hope that other research group will consider it as a case study.

\paragraph{{\bf Related works}}
In~\cite{ggr18}, we introduced the scheduling problem considered here but made the assumption that the underlying distributions of the tasks are known. We drop this assumption here and provide learning algorithms.
In~\cite{shields}, the framework to combine safety via shielding and model-free reinforcement learning is introduced and applied to several examples using table-based Q-learning as well as deep RL.
% In~\cite{ABCHKP19}, this is extended to quantitative objectives
% and the questions regarding ``minimal interference'' of a shield are further explored. 
%\track{Our shielded version of deep $Q$-learning fits the framework of \emph{post-posed} shielding of~\cite{shields}. 
% \track{
In~\cite{ABCHKP19}, shield synthesis is studied for long-run objective guarantees instead of safety requirements.
%The shield is synthesised by playing a two-player zero-sum game between a \emph{controller} player and the shield player.
%However, no formal argument is provided for considering the shield an adversary to the controller.
%Further, in both \cite{shields} and \cite{ABCHKP19}, MDPs are used to represent an abstraction of the unknown system on which the controller and the shield act, and the transition probabilities model lack of knowledge of the actual system.
Unlike our work, the transition probabilities on MDPs in both \cite{shields} and \cite{ABCHKP19} are assumed to be known.
We observe that \cite{shields} and \cite{ABCHKP19} do not provide model-based learning and PAC guarantees.
% but only convergence in some cases. 
While some pre-shielding literature does consider unknown MDPs (see, e.g.\cite{ft14}), we are not aware of PAC-learning works that focus on scheduling problems.
% } 
%Further, our results on the the model-based learning is crucial on the fact that the underlying MDP has a single maximal end-component.}
% which is important for our solution to model-based learning.

% \track{
In~\cite{KPR18}, we studied a framework to mix reactive synthesis 
%(for general omega-regular objectives and not only safety) 
and model-based reinforcement learning for mean-payoff with PAC guarantees.
%has been studied in~\cite{KPR18}.
%, but the approach that we follow in this paper is different from the one used in~\cite{KPR18}.
There, the learning algorithm estimates the probabilities on the transitions of the MDP. In our approach, we do not estimate these probabilities directly from the MDP, but learn probabilities for the individual tasks in the task system.
%that in turn gives an estimate of the probabilities on the transitions in the MDP. 
%This leads to a faster learning in our case. 
%Besided no implementation has been reported in~\cite{KPR18}.
% While this work is applicable to our scheduling problem, 
% it relies on the huge MDP underlying the task system and is thus expected to severely suffer from the state explosion problem. 
The efficient PAC guarantees that we have obtained for the model-based part cannot be obtained from that framework.
% }
%
Finally, in~\cite{CNPRZ17} we introduced a first combination of shielding with model-predictive control using MCTS, but did not consider learning.

\section{Preliminaries}\label{sec:preliminaries}
% \todo{We need good examples!}  
We denote by $\mathbb{N}$ the set of
natural numbers; by $\mathbb{Q}$, the set of rational numbers; and by
$\mathbb{Q}_{\geq 0}$ the set $\{q \in \mathbb{Q} \sth q \geq 0\}$ of
all non-negative rational numbers.
Given $n \in \mathbb{N}$, we denote by $[n]$ the set $\{1, \dots, n\}$.
% , and by $[n]_0$ the set $\{0, \dots, n\}$.
% \todo{Daddy cool: are these sets both supposed to be the same?}
Given a finite set $A$, a
(rational) \textit{probability distribution} over $A$ is a function
$\dist \colon A \rightarrow [0, 1] \cap \rat$ such that
$\sum_{a\in A} \dist(a) = 1$. 
We call $A$ the \emph{domain of $p$}, and denote it by $\dom(p)$.
We denote the set of probability distributions on $A$ by $\dists(A)$. The \textit{support} of the
probability distribution $\dist$ on $A$ is
$\supp(\dist) = \left\lbrace a \in A \;\vert\; \dist(a) >
  0\right\rbrace$. A distribution is called \emph{Dirac} if
$|\supp(\dist)| = 1$.
For a probability distribution $p$, the minimum probability assigned by $p$ to the elements in $\supp(p)$ is $\pi^p_{\min} = \displaystyle{\min_{a \in \supp(p)}(p(a))}$.
% The size of a distribution $p$ is the size of its support, that is, $|\supp(\dist)|$.
We say two distributions $\dist$ and $\dist'$ are \emph{structurally identical} if $\supp(\dist) = \supp(\dist')$.
Given two structurally identical distributions $\dist$ and $\dist'$, for $0 < \epsilon < 1$, we say that $\dist$ is $\epsilon$-close to $\dist'$, denoted $\dist \eclose{\epsilon}{} \dist'$, if $\supp(\dist)=\supp(\dist')$, and for all $a \in \supp(\dist)$, we have that $|\dist(a) - \dist'(a)| \le \epsilon$.

\paragraph{Scheduling problem} 
% Let us define the main problem we will study. 
An instance of the scheduling problem studied in \cite{ggr18} consists
of a task system $\Upsilon=((\tau_i)_{i \in [n]}, F, H)$, where
$(\tau_i)_{i \in [n]}$ are $n$ preemptible tasks partitioned into hard
and soft tasks $H$ and $F$ respectively. The latter need to be
scheduled on a \emph{single processor}.  \stam{ All tasks $\tau_i$ are
  identified with their unique respective index $i$.  Time is assumed
  to be discrete and measured in CPU ticks.  Each task $\tau_i$
  generates an infinite number of instances $\tau_{i,j}$, called
  \emph{jobs}, with $j=1,2, \dots$
% We assume that 
All tasks are either \emph{hard} or \emph{soft}, and $F$ and $H$ denote the set of indices of soft and hard tasks respectively (i.e. $i\in F$ iff $\tau_i$ is a soft task). 
Formally, a task system of $n$ tasks is denoted by the tuple $((\tau_i)_{i \in [n]}, F, H)$.
Jobs generated by both hard and soft tasks are equipped with deadlines,
which are relative to the respective arrival times of the jobs in the
system.
Jobs generated by hard tasks must complete before their
respective deadlines, but this is not mandatory for jobs generated by
soft tasks (although a deadline miss incurs a penalty in this case).
The tasks are also assumed to be independent, i.e. scheduling a job
of one task does not depend on another job belonging to some other
task. The scheduling problem consists in finding a
\emph{scheduler}, i.e.~a function that associates, to all CPU ticks,
the identity of the task that must run at that moment; in order to:
\begin{inparaenum}[(i)]
\item avoid deadline misses by hard tasks; and
\item minimise the mean-cost of deadline misses by soft tasks.  
\end{inparaenum}
}
Formally, the work of \cite{ggr18} relies on a probabilistic model for
the computation times of the jobs and for the delay between the
arrival of two successive jobs of the same task. For all $i \in [n]$,
task $\tau_i$ is defined as a tuple % $\zug{I_i, \C_i, D_i, \A_i}$,
$\zug{\C_i, D_i, \A_i}$,
where:
\begin{inparaenum}[(i)]
%\item 
% $I_i$ is a probability distribution with (finite)
%     support denoting the time at which the first job
%     of task $\tau_i$ arrives;
% $I_i$ denotes the time at which the first job
%     of task $\tau_i$ arrives;
\item $\C_i$ is a discrete probability distribution on the (finitely
  many) possible computation times of the jobs generated by $\tau_i$;
\item $D_i\in \N$ is the deadline of all jobs generated by $\tau_i$
  which is relative to their arrival time; and % of the jobs; and
\item $\A_i$ is a discrete probability distribution on the (finitely
  many) possible inter-arrival times of the jobs generated by $\tau_i$.
\end{inparaenum}
We denote by $\pi_{\max}^\Upsilon$ the maximum probability appearing in the definition of $\Upsilon$, that is, across all the distributions $\C_i$ and $\A_i$, for all $i \in [n]$.
It is assumed that
% $\max(\supp(\C_i)) \leq D_i \le \min(\supp(\A_i))$ 
$\max(\dom(\C_i)) \leq D_i \le \min(\dom(\A_i))$ for all
$i \in [n]$; hence, at any point in time, there is at
most one job per task in the system.
Also note that when a new job of some task arrives at the system, the deadline for the previous job of this task is already over.
Finally, we assume that the task system is \emph{schedulable for the hard tasks}, meaning that it is possible to guarantee that jobs associated to hard tasks never miss their deadlines. 
On the other hand, the full set of tasks may not be schedulable, so that jobs associated with soft tasks may be allowed to miss their deadlines.
The potential degradation in the quality when a soft task misses its deadline is modelled by a cost function
$cost: F \rightarrow \mathbb{Q}_{\ge 0}$ that associates to each
soft task $\tau_j$ a cost $c(j)$ that is incurred every time a job
of $\tau_j$ misses its deadline. \label{def:EDF} As a final observation, we recall the \emph{earliest deadline first} (EDF) algorithm that always gives execution time to the job closest to its deadline. EDF is an optimal scheduling algorithm in the following sense: if a task system is schedulable (without any misses at all) then EDF will yield such a feasible schedule~\cite{buttazzo11}. 
In general, applying EDF on both the hard and soft tasks may cause hard tasks to miss deadlines, as the entire task system may not be schedulable.
However, one may apply EDF on hard tasks only, and allow for soft tasks whenever no hard task is available.
This version of EDF ensures that all jobs of hard tasks are scheduled in time, but
% is not ideal since it may make the system infeasible.
% In the sequel, we will compare against EDF applied on the hard tasks only to ensure they are scheduled if possible. 
does not guarantee optimality with respect to cost.

\stam{
\todo[inline]{This part must moved elsewhere, where we clarify the
  hypothesis of the algorithms. G.}
Finally, assume that we know a lower
bound on the inter-arrival and execution-time probabilities as well as
a global maximal inter-arrival time. That is, we are given
$\pmin \in \mathbb{Q} \cap [0,1]$ and $\Omega \in \mathbb{N}$ such
that, for all $1 \leq i \leq n$ and all $k \in \mathbb{N}$:
$\A_i(k) > 0 \implies \A_i(k) \geq \pmin$;
$\C_i(k) > 0 \implies \C_i(k) \geq \pmin$; and
$\max (\A_i) \leq \Omega$.
\todo[inline]{until here}
}

% \paragraph{Assumptions.}
% Henceforth we make the following assumptions regarding task systems.
% \begin{enumerate}
%   \item Time is discrete (and measured in clock ticks).
%   \item For all $1 \leq i \leq n$ we have $\max C_i \leq D_i \leq \min A_i$,
%     and therefore, at any point in time the system has one job per task at
%     most.
%   \item The scheduler can \emph{preempt} jobs. That is, at every clock tick,
%     it can assign execution time to any pending job.
%   \item We know a lower bound on inter-arrival and execution-time
%     probabilities as well as a global maximal inter-arrival time. In other
%     words, we are given $\pmin \in \mathbb{Q} \cap [0,1]$ and $\Omega \in
%     \mathbb{N}$ such that $A_i(k) > 0 \implies A_i(k) \geq \pmin$,
%     $C_i(k) > 0 \implies C_i(k) \geq \pmin$,
%     and $\max
%     A_i \leq \Omega$ for all $1 \leq i \leq n$ and all $k \in \mathbb{N}$.
%   \end{enumerate}

Given a task system
$\Upsilon=((\tau_i)_{i \in [n]}, F, H)$ with $n$ tasks, the structure
of $\Upsilon$ is $((\struct(\tau_i))_{i \in [n]}, F, H)$ where
$\struct(\zug{\mathcal{C}, D, \mathcal{A}}) = (\zug{\dom(\mathcal{C}),
  D, \dom(\mathcal{A})})$.
  %Given a task system
%$\Upsilon = ((\tau_i)_{i \in [n]},F,H)$ with $n$ tasks, 
We denote by
$\mathcal{C}_{\max}$ and $\mathcal{A}_{\max}$ resp. the maximum
computation time, and the maximum inter-arrival time of a task in
$\Upsilon$.  Formally,
$\mathcal{C}_{\max}=\max(\bigcup_{i \in [n]}\dom(\mathcal{C}_i))$, and
$\mathcal{A}_{\max} = \max(\bigcup_{i \in [n]}
\dom(\mathcal{A}_i))$.
% and we define $\mathbb{M} = \max(\mathbb{C}, \mathbb{A})$.
Note that $\mathcal{A}_{\max} \ge \mathcal{C}_{\max}$.
% \sgcomment{We do not use $\mathbb{C}$ anywhere.}
We also let $\mathbb{D} = {\max_{i \in [n]}(|\dom(\mathcal{A}_i)|)}$.
We denote by $|\Upsilon|$ the number of tasks in the task system $\Upsilon$.
Consider two task systems $\Upsilon_1=((\tau_i^1)_{i \in [n]}, F, H)$, and $\Upsilon_2=((\tau_i^2)_{i \in [n]}, F, H)$, with $|\Upsilon_1| = |\Upsilon_2|$, $\tau_i^j = \zug{\C^j_i, D^j_i, \A^j_i}$ for all $i \in [n]$ and $j \in [2]$.
The two task systems $\Upsilon_1$ and $\Upsilon_2$ are said to be \emph{$\epsilon$-close}, denoted $\Upsilon_1 \approx^\epsilon \Upsilon_2$, if
\begin{inparaenum}[(i)]
  \item $\struct(\Upsilon^1) = \struct(\Upsilon^2)$,
  \item for all $i \in [n]$, we have $\mathcal{A}^1_i \sim^\epsilon \mathcal{A}^2_i$, and
  \item for all $i \in [n]$, we have $\mathcal{C}^1_i \sim^\epsilon \mathcal{C}^2_i$.
\end{inparaenum}

\paragraph{Markov decision processes} Let us now introduce
\emph{Markov Decision Process} (MDP) as they form the basis of the
formal model of~\cite{ggr18}, which we recall later.  A finite
\emph{Markov decision process} is a tuple
$\Gamma=\zug{V,E,L,(V_\playerOne, V_\playerTwo), A, \prob, \cst}$,
where:
\begin{inparaenum}[(i)]
\item $A$ is a finite set of actions;
\item $\zug{V,E}$ is a finite directed graph
%with set of vertices $E$, set of
  %edges $E$ and 
  and $L$ is an edge-labelling function % of the edges $L$ 
  (we denote by
  $E(v)$ the set of outgoing edges from vertex $v$);
\item the set of vertices $V$ is partitioned into $V_\playerOne$ and
  $V_\playerTwo$; 
\item the graph is bipartite i.e.
  $E\subseteq (V_\playerOne\times V_\playerTwo)\cup
  (V_\playerTwo\times V_\playerOne)$, and the labelling
  function is
  s.t. $L(v,v')\in A$ if $v\in V_\playerOne$, and $L(v,v')\in \rat$ if
  $v\in V_\playerTwo$; and
\item $\prob$ assigns to each vertex $v \in V_\ocircle$ a rational
  probability distribution on $E(v)$. 
\end{inparaenum}
For all edges $e$, we let $\cst(e)=L(e)$ if $L(e)\in\rat$, and
$\cst(e)=0$ otherwise.  We further assume that, for all
$v\in V_\playerOne$, for all $e$, $e'$ in $E(v)$: $L(e)=L(e')$ implies
$e=e'$, i.e. an action identifies uniquely an outgoing edge.
Given $v \in V_\Box$, and $a \in A$, we define $\post(v,a) = \lbrace
v'\in V_\playerTwo \mid (v,v')\in E \textrm{ and } L(v,v')=a \rbrace
\cup \lbrace v'' \in V_\playerOne \mid \exists v':(v,v')\in E, L(v,v')=a\text{ and } \delta(v',v'') > 0 \rbrace$. 
For all vertices $v \in V_\Box$, we denote by $A(v)$, the set of actions $\lbrace a \in A \mid \post(v,a) \cap V_\playerOne \neq \emptyset \rbrace$.
The size of an MDP $\Gamma$, denoted $|\Gamma|$, is the sum of the number of vertices and the number of edges, that is, $|V|+|E|$.
An MDP $\Gamma=\zug{V,E,L,(V_\playerOne, V_\playerTwo), A, \prob, \cst}$ is said to \emph{structurally identical} to another MDP $\Gamma'=\zug{V,E,L',(V_\playerOne, V_\playerTwo), A, \prob', \cst}$ if for all $v \in V_\ocircle$, we have that $\supp(\prob(v)) = \supp(\prob'(v))$.
For two structurally identical MDPs $\Gamma$ and $\Gamma'$ with distribution assignment functions $\delta$ and $\delta'$ respectively, we say that $\Gamma$ is $\epsilon$-approximate to $\Gamma'$, denoted $\Gamma \approx^\epsilon \Gamma'$, if for all $v \in V_\ocircle$:
% and for all $(v,v') \in \supp(\prob(v))$, we have that $|\prob((v,v')) - \prob'((v,v'))| \le \epsilon$.
$\delta(v) \sim^\epsilon \delta'(v)$.

An MDP $\Gamma$ can be interpreted as a game $\mathcal{G}_\Gamma$
between two players: $\playerOne$ and $\playerTwo$, who own the vertices in $V_\Box$ and
$V_\ocircle$ respectively. A play in an MDP is a path in its
underlying graph $\zug{V,E, A\cup\rat}$. We say that a prefix
$\play(n)$ of a play $\play$ belongs to player
$i\in \{\playerOne, \playerTwo\}$, iff its last vertex
$\last{\play(n)}$ is in $V_{i}$. The set of prefixes that belong to
player $i$ is denoted by $\prefixesPlayer_i{\mathcal{G}_\Gamma}$. A
play is obtained by the interaction of the players: if the current play prefix $\play(n)$ belongs to
$\playerOne$, she plays by picking an edge $e\in E(\last{\play(n)})$
(or, equivalently, an action that labels a necessarily unique edge
from $\last{\play(n)}$). Otherwise, when $\play(n)$ belongs to
$\playerTwo$, the next edge $e\in E(\last{\play(n)})$ is chosen
randomly according to $\prob(\last{\play(n)})$. In both cases, the
plays prefix is extended by $e$ and the game goes \textit{ad infinitum}.

A (deterministic) \emph{strategy} of $\playerOne$ is a function
$\sigma_\Box: \prefixesPlayer_\Box{\mathcal{G}} \rightarrow E$, such that
$\sigma_\Box(\prefix) \in E(\last{\prefix})$ for all prefixes.  A
strategy $\sigma_\Box$ is \emph{memoryless} if for all finite prefixes
$\prefix_1$ and $\prefix_2 \in \prefixes{\graph}$:
$\last{\prefix_1}=\last{\prefix_2}$ implies
$\sigma_\Box(\rho_1) = \sigma_\Box(\rho_2)$. For memoryless
strategies, we will abuse notations and assume that such strategies
$\sigma$ are of the form $\sigma: V_\playerOne\rightarrow E$
% \todo{Again mapsto here} 
(i.e., the strategy associates the edge to play to the current vertex and not to
the full prefix played so far). From now on, we will consider
memoryless deterministic strategies unless otherwise stated.
Let $\Gamma=\zug{V,E,L,(V_\Box, V_\ocircle),A,\prob,\cst}$ be an MDP, and
let $\sigma_\Box$ be a \emph{memoryless} strategy. Then, assuming that
$\playerOne$ plays according to $\sigma_\playerOne$, we can express
the behaviour of $\Gamma$ as a Markov chain $\Gamma[\sigma_\Box]$, where the
probability distributions reflect the stochastic choices of
$\playerTwo$ (see \cite{ggr18} for the details).
\stam{Formally,
$\Gamma[\sigma_\playerOne]=\zug{V_\playerTwo, E', L', \prob'}$, where
$(v,v')\in E'$ iff there is $\hat{v}$ s.t.:
\begin{inparaenum}[(i)]
\item $(v,\hat{v})\in E$;
\item $\sigma_\playerOne(\hat{v})=v'$; and
\item $\prob(\hat{v},v')=\prob'(v,v')$.
\end{inparaenum}
Further, for all $e \in E'$, we have $L'(e) = L(e)$.
}

\paragraph{End components} 
% Let $\Gamma=\zug{V,E,L,(V_\playerOne, V_\playerTwo), A, \prob, \cst}$ be
% an MDP, and let $(V',E')$ be a pair s.t. $V'\subseteq V$ and
% $E'\subseteq E\cap (V_\playerOne\times V_\playerTwo)$ (i.e., $E'$
% contains only edges from Player one's nodes). Then, $(V',E')$ is an
% \emph{end component} (EC for short) iff:
% \begin{inparaenum}[(i)]
% \item $(V',E'')$ with
%   $E''=E'\cup \big(E\cap (V_\playerTwo\times V_\playerOne)\big)$ is a
%   \emph{strongly connected} subgraph of $(V,E)$; and
% \item whenever Player one plays an edge $e\in E'(v)$ from some
%   $v\in V'$, then the next Player one state the MDP reaches is
%   necessarily in $V'$ too (whatever the probabilistic choice of Player
%   two). Formally, for all $v\in V'$, for all $e\in E'(v)$:
%   $V'\supseteq \{\trg(e')\mid e'\in \supp(\delta(\trg(e)))\}$.
% \end{inparaenum}
% Then, we say that an EC is \emph{maximal} (denoted MEC) iff it is
% maximal for the graph inclusion ordering $\subseteq$.
% \todo{Another definition below, possibly simpler: S}
An \emph{end-component} (EC) $M = (T,A')$, with
$T \subseteq V$ and $A':T \cap V_\Box \rightarrow 2^A$, is a
\emph{sub-MDP} of $\Gamma$ such that: for all $v \in T \cap V_\Box$, $A'(v)$ is a
subset of the actions available to $\playerOne$ from $v$; for all
$a \in A'(v)$, $\post(v,a) \subseteq T$; and, it's underlying graph is strongly connected. A
\emph{maximal end-component} (MEC) is an EC that is not included in
any other EC.

\paragraph{MDP for the scheduling problem}
%Let us now sketch the model from \cite{ggr18}. 
Given a system
$\Upsilon=\{\tau_1,\tau_2,\ldots,\tau_n\}$ of tasks, we describe below
the modelling of the scheduling problem by an MDP
$\Gamma_\Upsilon=\zug{V,E,L,(V_\Box, V_\ocircle),A, \delta, \cst}$ as
it appears in \cite{ggr18}. The two players $\playerOne$ and
$\playerTwo$ correspond respectively to the \emph{Scheduler} and the
task generator (\emph{TaskGen}) respectively. %The vertices
%correspond to the system states. 
Since %Remember that 
there is at most one
job per task that is active at all times, %. Thus,
vertices encode
%we maintain, in
%all vertices, 
the following information about each task $\tau_i$:
\begin{inparaenum}[(i)]
\item a \emph{distribution} $c_i$ over the job's possible remaining
  computation times (rct);
\item the time $d_i$ up to its deadline; and
\item a distribution $a_i$ over the possible times up to the next
  arrival of a new job.
  %of $\tau_i$.
\end{inparaenum}
We also tag vertices with either $\playerOne$ or $\playerTwo$ to
remember their respective owners and we have a vertex $\bot$
that is reached when a hard task misses a deadline.
\stam{Formally:
  $V_\playerOne=\big(\D([\Cmax]_0)\times
  [\Dmax]_0\times\D([\Amax]_0)\big)^n\times\{\playerOne\}\cup
  \{\bot\}$ and
  $V_\playerTwo=\big(\D([\Cmax]_0)\times
  [\Dmax]_0\times\D([\Amax]_0)\big)^n\times\{\playerTwo\}$; where
%$\Cmax=\max_i(\supp(\C_i))$, 
$\Cmax=\max_i(\max(\supp(\C_i)))$, 
$\Dmax=\max_i(\{D_i\})$ and
%$\Amax=\max_i(\supp(\A_i))$.
$\Amax=\max_i(\max(\supp(\A_i)))$.
%
% \begin{align*}
%   V_\playerOne &= \big(\dist([\Cmax]_0)\times
%                  [\Dmax]_0\times\dist([\Amax]_0)\big)^n\times\{\playerOne\}\\
%   V_\playerTwo &= \big(\dist([\Cmax]_0)\times
%                  [\Dmax]_0\times\dist([\Amax]_0)\big)^n\times\{\playerTwo\}\cup \{\bot\}.
% \end{align*}
%
We denote by $\vec{\mathcal{C}}$ the set of computation time distributions, by $\vec{D}$ the set of deadlines, and by $\vec{\mathcal{A}}$ the set of inter-arrival time distributions for all tasks.
We also use $\vec{\mathcal{C}}_H, \vec{D}_H, \vec{\mathcal{A}}_H$ and $\vec{\mathcal{C}}_F, \vec{D}_F, \vec{\mathcal{A}}_F$ to denote these sets for the hard and the soft tasks respectively.
We denote by $\vec{\mathcal{C}}$ the set of computation time distributions, by $\vec{D}$ the set of deadlines, by $\vec{\mathcal{A}}$ the set of inter-arrival distributions for all tasks.
Similarly, we also use the notations $\vec{\mathcal{C}}_i$, $\vec{D}_i$, and $\vec{\mathcal{A}}_i$ for $i \in \{F, H\}$ corresponding to the set of soft and hard tasks respectively.
}
For a vertex $v=\big((c_1, d_1, a_1)\ldots (c_n, d_n, a_n), \Delta \big)$, for $\Delta \in \{\Box, \ocircle\}$, let $\act(v)=\{i\mid c_i(0)\neq 1\text{ and }d_i>0\}$ be the tasks that have an active job in $v$; $\mis(v)=\{i\mid c_i(0)=0\text{ and } d_i=0\}$, those that have missed a deadline in $v$.
% Intuitively, $\act(v)$
% is the set of tasks that have an active job in $v$, that is one which
% has not finished and whose deadline has not passed yet; and $\mis(v)$
% is the set of tasks that have missed a deadline \emph{for sure} in $v$.
% For $i \in [n]$, we denote by $\proj(v,i)$, the tuple corresponding to the $i$-th task, that is, $(c_i,d_i,a_i)$.
% For a task $\tau_i$, we denote by $\proj(v,\tau_i)$, the tuple corresponding to the task $\tau_i$, that is of the form $(c_i,d_i,a_i)$.

\stam{
\textbf{Distribution updates.} We now introduce the $\decr$ and
$\upd$ functions that will be useful when we will need to update the
knowledge of the Scheduler. For example, consider a state where
$c_i(1)=0.5$, $c_i(4)=0.1$ and $c_i(5)=0.4$ for some $i$, and where
$\tau_i$ is granted one CPU time unit. Then, all elements in the
support of $c_i$ should be decremented, yielding $c'_i$ with
$c_i'(0)=0.5$, $c_i'(3)=0.1$ and $c_i'(4)=0.4$. Since
$0\in\supp(c_i')$, the current job of $\tau_i$ \emph{could} now
terminate with probability $c_i'(0)=0.5$, or continue running, which
will be observed by the Scheduler player. In the case where the job does not
terminate, the probability mass must be redistributed to update
Scheduler's knowledge, yielding the distribution $c_i''$ with
$c_i''(3)=0.2$ and $c_i''(4)=0.8$.
}

\paragraph{Possible moves} The possible actions of
Scheduler are to schedule an active task or to idle the
CPU. We model this by having, from all vertices $v\in V_\playerOne$ one
transition labelled by some element from $\act(v)$, or by
$\varepsilon$.
% \sgcomment{We use $\varepsilon$ now in other contexts as well.}. 
%Such transitions model the elapsing of one clock tick.
The moves of TaskGen consist in
selecting, for each task one possible \emph{action} out of four:
either
\begin{inparaenum}[(i)]
\item  nothing ($\varepsilon$); or
\item  to finish the current job without
  submitting a new one ($fin$); or 
\item to submit a new job while the previous one is already finished
  ($sub$); or
\item to submit a new job and kill the previous one, in the case of a
  soft task ($killANDsub$), which will incur a cost.
\end{inparaenum}

We consider the following example from \cite{ggr18}.
\begin{figure}[h!]
\vspace*{-0.2cm}
  \centering
  \scalebox{.8}{
  \begin{tikzpicture}
    \node[player1,initial,initial text={}] (vinit) at (0,0)
    {$\textbf{(1,2,3)}$\\$\textbf{([1:.4,2:.6],2,3)}$} ;

    \node[player2] (v1) at (-4, -1.7) {$\textbf{(1,1,2)}$\\$\textbf{([0:.4,1:.6],1,2)}$} ;
    \node[player2] (v2) at (0, -1.7) {${(0,1,2)}$\\$\textbf{([1:.4,2:.6],1,2)}$} ;
    \node[player2] (v3) at (4, -1.7)
    {$\textbf{(1,1,2)}$\\$\textbf{([1:.4,2:.6],1,2)}$} ;

    \path[-latex]
    (vinit) edge node[left=.3cm] {$s$} (v1)
            edge node[left] {$h$} (v2)
            edge node[right=.3cm] {$\varepsilon$} (v3)
            ;
    \node[player1] (v4) at (-7, -3.4) {$\textbf{(1,1,2)}$\\$(0,1,2)$} ;
    \node[player1] (v5) at (-4, -3.4) {$\textbf{(1,1,2)}$\\$\textbf{(1,1,2)}$} ;
    \node[player1] (v6) at (0, -3.4)
    {$(0,1,2)$\\$\textbf{([1:.4,2:.6],1,2)}$} ;
    \node[player1] (v7) at (4, -3.4)
    {$\textbf{(1,1,2)}$\\$\textbf{([1:.4,2:.6],1,2)}$} ;
    
    \path[-latex]
    (v1) edge node[left=.3cm]{$(\varepsilon, fin)$} node[right]{$.4$} (v4)
         edge node[left]{$(\varepsilon,\varepsilon)$} node[right]{$.6$} (v5)
         % shibashis
%    (v2) edge node[left]  {$(\varepsilon,\varepsilon)$} (v6)
    (v2) edge node[left]  {$(fin,\varepsilon)$} (v6)
    (v3) edge node[right]  {$(\varepsilon,\varepsilon)$} (v7)
    ;

    \node[player2] (v8) at (-7, -5.1) {$(0,0,1)$\\$(0,0,1)$} ;
    \node[player1] (v9) at (-7, -6.8) {$(0,0,1)$\\$(0,0,1)$} ;
    \node[player2] (v10) at (-6.5, -1) {$(0,0,0)$\\$(0,0,0)$} ;

    \path[-latex]
    (v4) edge node[left] {$h$} (v8)
    % shibashis
%   (v8) edge node[right] {$(\varepsilon,\varepsilon)$} (v9)
   (v8) edge node[right] {$(fin,\varepsilon)$} (v9)
    (v9) edge[bend left=45] node[left] {$\varepsilon$} (v10)
    (v10) edge node[above] {$(sub,sub)$} (vinit)
    ;
    
    \node[player2] (v11) at (-4, -5.1) {$(0,0,1)$\\$\textit{(1,0,1)}$} ;
    \node[player1] (v12) at (-4, -6.8) {$(0,0,1)$\\$\textit{(1,0,1)}$} ;
    \node[player2] (v13) at (-2, -6.8) {$(0,0,0)$\\$\textit{(1,0,0)}$} ;
    
    \path[-latex]
    (v5) edge node[right] {$h$} (v11)
         edge[dotted] node[left=.2cm] {$s$} (-5, -4.4)
         edge[dotted] node[right=.2cm] {$\varepsilon$} (-3, -4.4)
         % shibashis
   % (v11) edge node[right] {$(\varepsilon,\varepsilon)$} (v12)
    (v11) edge node[right] {$(fin,\varepsilon)$} (v12)
    (v12) edge node[above] {$\varepsilon$} (v13)
    (v13) edge[out=90, in=-150] node[pos=.1,right,align=left] {$(sub,killANDsub)$\\\textbf{cost=10}} (vinit)
    ;

   \path[-latex]
   (v6) edge[dotted] node[left=.2cm] {$s$} (-1, -4.4)
        edge[dotted] node[right=.2cm] {$\varepsilon$} (1, -4.4)
   ;

   \node[player2] (v14) at (4,-5.1)
   {$\textit{(1,0,1)}$\\$\textit{([1:.4,2:.6],0,1)}$} ;
   \node[player1] (bot) at (4, -6.8) {$\bot$} ;
   
   \path[-latex]
   (v7) edge node[right] {$\varepsilon$} (v14)
        edge[dotted] node[left=.2cm] {$s$} (3, -4.4)
        edge[dotted] node[right=.2cm] {$h$} (5, -4.4)
   (v14) edge node[right] {$\varepsilon$} (bot)
   (bot) edge[loop right] node[right] {$\varepsilon$} (bot)
;
   
  \end{tikzpicture}
}
  \caption{\label{fig-example} MDP excerpt for Ex.~\ref{exm-prelims}. \textbf{Bold} tasks are active,
    those in \textit{italics} have missed a deadline.}
\end{figure}
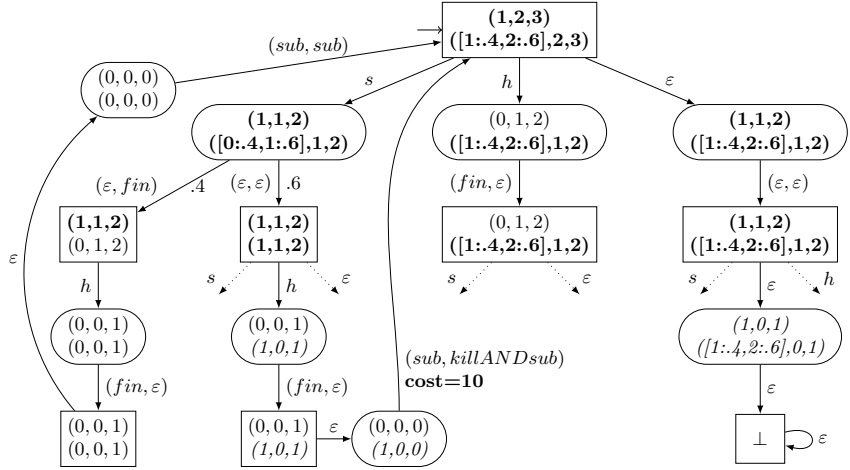

\begin{example} \label{exm-prelims}
  Consider a system with one hard task
  $\tau_h=\zug{\C_h,2,\A_h}$ s.t. $\C_h(1)=1$ and $\A_h(3)=1$; one
  soft task $\tau_s=\zug{\C_s, 2, \A_s}$ s.t. $\C_s(1)=0.4$,
  $\C_s(2)=0.6$, and $\A_s(3)=1$; and the cost function $c$
  s.t. $c(\tau_s)=10$. 
%   This means that both tasks will submit their
%   first job at time 0, both with deadlines at time $0+2=2$. Then,
%   $\tau_{h,1}$ will have a computation time of~$1$, while $\tau_{s,1}$
%   will have a computation time which is either~$1$ (with probability
%   $0.4$) or $2$ (with probability $0.6$). Both tasks will submit new
%   jobs $\tau_{h,2}$ and $\tau_{s,2}$ at time $0+3=3$. Each time a job
%   of $\tau_s$ misses its deadline, a cost of $10$ will be incurred.
%   \figurename
  Fig.~\ref{fig-example} presents an excerpt of the MDP
  $\Gamma_\Upsilon$ built from the set of tasks $\tau=\{\tau_h,\tau_s\}$
  of Example~\ref{exm-prelims}. A distribution $p$ with
  support $\{x_1,x_2,\ldots,x_n\}$ is denoted by
  $[x_1:p(x_1),x_2:p(x_2),\ldots;x_n:p(x_n)]$. When $p$ is
  s.t. $p(x)=1$ for some $x$, we simply denote $p$ by $x$. Vertices from
  $V_\playerOne$ and $V_\playerTwo$ are depicted by rectangles and
  rounded rectangles respectively. Each vertex is labelled by
  $(c_h,d_h,a_h)$ on the top, and $(c_s,d_s,a_s)$ below.

  A strategy to avoid missing a deadline of $\tau_h$ consists in first
  scheduling $\tau_s$, then $\tau_h$. One then reaches the left-hand
  part of the graph from which $\playerOne$ can avoid $\bot$ whatever
  $\playerTwo$ does. Other safe strategies are possible: the
  first step of the algorithm in \cite{ggr18} is to compute all the
  \emph{safe} nodes (i.e. those from which $\playerOne$ can ensure to
  avoid $\bot$), and then find an optimal one w.r.t to missed-deadline costs.
  %cost
  %of missing deadlines for soft tasks) among those.
  
  There are two optimal
  memoryless strategies, one in which Scheduler
  first chooses to execute $\tau_h$, then $\tau_s$; and another where
  $\tau_s$ is scheduled for $1$ time unit, and then preempted to let
  $\tau_h$ execute.  Since the time difference between the arrival of two consecutive jobs of the soft task
%   period of
$\tau_s$ is $3$ and the cost
  of missing a deadline is $10$, for both of these optimal strategies,
  the soft task's deadline is missed with probability $0.6$ over this time duration of $3$,
%   during each period 
  and hence the mean-cost is~$2$. There is
  another safe schedule that is not optimal which only grants $\tau_h$ is
  CPU access, and never schedules $\tau_s$, thus giving a
  mean-cost of $\frac{10}{3}$.  
  \qed
\end{example}

\paragraph{Expected mean-cost}  Let us first
associate a value, called the \emph{mean-cost} $\mpay(\play)$ to all
plays $\play$ in an MDP
$\Gamma=\zug{V,E,L,(V_\Box, V_\ocircle),A,\prob,\cst}$. 
First, for a prefix $\prefix=e_0e_1\ldots e_{n-1}$, we define
$\mpay(\prefix) = \frac{1}{n} \sum_{i = 0}^{i = n-1} \cst(e_{i})$
(recall that $\cst(e)=0$ when $L(e)$ is an action). Then, for a play
$\play=e_0e_1\ldots$, we have
$\mpay(\play) = \limsup_{n \rightarrow \infty} \mpay (\play(n))$.
Observe that $\mpay$ is a measurable function.  
A strategy $\sigma_{\Box}$ is {\em optimal} for the mean-cost from some initial vertex $\initV \in V_\Box$ if $\expect_{\initV}^{\Gamma[\sigma_\Box]}(\mpay)=\inf_{\sigma_{\Box}'}\expect_{\initV}^{\Gamma[\sigma_\Box']}(\mpay)$. Such {\em optimal} strategy always exists, and it is well-known that there is always one which is \emph{memoryless}. Moreover, this problem can be solved in polynomial time through linear programming~\cite{FV97} or in practice using value iteration (as implemented, for example, in the tool {\sc Storm} \cite{DJKV17}). We denote by $\expect_{\initV}^{\Gamma}(\mpay)$ the optimal value $\inf_{\sigma_{\Box}}\expect_{\initV}^{\Gamma[\sigma_\Box]}(\mpay)$.

\paragraph{Safety synthesis} 
Given an MDP $\Gamma = \zug{V,E,L,(V_\Box, V_\ocircle),A,\prob,\cst}$,
an initial vertex $\initV \in V$, and a strategy $\sigma_\Box$, we
define the set of possible \textit{outcomes} in the Markov chain
$\Gamma[\sigma_\Box]$ as the set of paths
$\initV=v_0 v_1 v_2 \ldots{}$ in $\Gamma[\sigma_\Box]$ s.t., for all
$i\geq 0$, there is non-null probability to go from $v_i$ to $v_{i+1}$ in
$\Gamma[\sigma_\Box]$.
%
% as
% $\outcomesMC{\Gamma[\sigma_\Box]}{\initV} = \{ \play = e_0 e_1 e_2 \ldots{} \in
% \plays{\graph} \mid \src(e_0) = \initV \wedge (\forall\, n \in \nat,
% \, e_{n+1}\in\supp(\trg(e_n))\}$. 
Let
$V_{\outcomesMC{\Gamma[\sigma_\Box]}{\initV}} \subseteq V$ denote the set of vertices
visited in the set of possible outcomes $\outcomesMC{\Gamma[\sigma_\Box]}{\initV}$.

% Given an MDP $\Gamma = \zug{V,E,L,(V_\Box, V_\ocircle),A,\prob,\cst}$,
Given $\Gamma$ with vertices $V$,
initial vertex $\initV \in V$, and a set $\badV \subseteq V$ of
\emph{bad vertices}, the \emph{safety synthesis problem} is to decide
whether $\playerOne$ has a strategy $\sigma_\Box$ ensuring to visit the
safe vertices only, i.e.:
$V_{\outcomesMC{\Gamma[\sigma_\Box]}{\initV}} \cap \badV=\emptyset$.
% \sgcomment{$V_{\outcomesMC{\Gamma[\sigma_\Box]}{\initV}}$ is not defined.}
%(such vertices will model hard-task deadline misses).
If this is the case, we call such a
strategy \emph{safe}. The safety synthesis problem is decidable in
polynomial time for MDPs
%Indeed, since probabilities do %not matter
%for this problem, the MDP can be regarded as a plain two-player game
%played on graphs (as in 
(see, e.g., safety games in~\cite{Thomas95}).
%, and the classical
%\emph{attractor} algorithm can be used.
% (see Appendix \ref{app:attractor}). 
Moreover, if a safe
strategy exists, then there is a \emph{memoryless} safe
strategy. Henceforth, we will consider safe strategies that are
memoryless only. We say that a vertex $v$ is safe iff $\playerOne$
has a safe strategy from $v$, and that an edge
$e=(v,v')\in E\cap (V_\playerOne \times V_\playerTwo)$ is safe iff there is a
safe strategy $\sigma_\playerOne$ s.t. $\sigma_\playerOne(v)=v'$. So,
the \emph{safe edges} $\safe(v)$ from some node $v$ correspond to the choices that
$\playerOne$ can safely make from $v$.
The set of safe edges exactly correspond to the set of safe actions that $\playerOne$ can make from $v$.
Then, we let the \emph{safe
  region} of $\Gamma$ be the MDP $\Gamma^{\text{safe}}$ obtained from
$\Gamma$ by applying the following transformations:
\begin{inparaenum}[(i)]
\item remove from $\Gamma$ all \emph{unsafe edges};
\item remove from $\Gamma$ all vertices and edges that are not reachable
  from $\initV$.
\end{inparaenum}
%Note
%Observe that  
%$\Gamma^{\text{safe}}$ is an MDP, since 
%every unsafe edge starts from a vertex of Player one.
% we have removed edges from Player one vertices only.

\paragraph{Most general safe scheduler}\label{sec:mgs-def}
Consider a task system $\Upsilon$ that is schedulable for the hard
  tasks.
  %, so that there exists a scheduler that
%prevents hard tasks from ever missing a deadline. 
Then, Scheduler has a winning strategy to avoid $\bot$ in
$\Gamma_\Upsilon$.
% A task system $\Upsilon$ is said to be \emph{schedulable for the hard
%   tasks} if Scheduler has a winning strategy to avoid $\bot$ in
% $\Gamma_\Upsilon$. 
% This strategy corresponds to a scheduler that
% prevents hard tasks from ever missing a deadline.  
We say a
non-deterministic strategy in $\Gamma_\Upsilon$ is the \emph{most general safe scheduler} (MGS) for the hard tasks
if from any vertex of Scheduler it allows all
 safe edges\footnote{The existence of a most general safe scheduler follows from the existence of a unique most general (a.k.a. maximally permissive) strategy for safety objectives \cite{RW87}.}.

\section{Model-Based Learning}
\label{sec:model_based}

We now investigate the case of \emph{model-based
  learning} of task systems. First, we consider the simpler case of
task systems with only soft tasks. We show that those systems are
always efficiently PAC learnable. Second, we consider learning task
systems with both hard and soft tasks. In that case, we study two
conditions for learnability. The first condition allows us to identify
task systems that are safely PAC learnable, i.e. learnable while
enforcing safety for the hard tasks. The second condition is stronger
and allows us to identify task systems that are safely and {\em
  efficiently} PAC learnable. 
  %Our learning algorithms on (safely)
%sampling the distributions underlying the behaviour of tasks.
  
\paragraph{{\bf Learning setting}}
We consider a setting in which we are given the structure of a task
system $\Upsilon=((\tau_i)_{i \in I},F,H)$ to schedule. While the
structure is known, the actual distributions that describe the
behaviour of the tasks are unknown and need to be learnt to behave
optimally or near optimally. The learning must be done only by
observing the jobs that arrive along time. When the task system
contains some hard tasks ($H\not=\emptyset$), all deadlines of such
tasks must be enforced.

For learning the inter-arrival time distribution of a task, a
\emph{sample} corresponds to observing the time difference between the
arrivals of two consecutive jobs of that task. For learning the
computation time distribution, a sample corresponds to observing the CPU
time  a job of the task has been assigned up to completion. Thus
if a job does not finish execution before its deadline, we do not
obtain a valid sample for the computation time.  Given a class
of task systems, we say:
\begin{itemize}
\item the class is \emph{probably approximately correct (PAC)
    learnable} if there is an algorithm $\mathbb{L}$ such that for
  all task systems $\Upsilon$ in this class, for all
  $\epsilon, \gamma \in (0,1)$: given $\struct(\Upsilon)$, the
  algorithm $\mathbb{L}$ can execute the task system $\Upsilon$, and
  can compute $\Upsilon^M$ such that
  $\Upsilon \approx^\epsilon \Upsilon^M$, with probability at
  least $1 - \gamma$.
\item the class is \emph{safely PAC learnable} if it is PAC
  learnable, and $\mathbb{L}$ can ensure safety for the hard tasks
  while computing $\Upsilon^M$.
\item the class is (safely) \emph{efficiently PAC
    learnable} if it is (safely) PAC learnable, and there is
  a polynomial $q$ in the size of the task system, in
  $\nicefrac{1}{\epsilon}$, and in $\nicefrac{1}{\gamma}$, s.t.
  $\mathbb{L}$ obtains enough samples to compute
  $\Upsilon^M$ in a time bounded by $q$.
\end{itemize}

Note that our notion of efficient PAC learning is stronger
than the definition used in classical PAC learning terminology
\cite{Valiant84} since we take into account the time that is needed to
get samples and not only the number of samples needed.
%We will see
%later in this section why this distinction is important.

\paragraph{{\bf Learning discrete finite distributions}}
\label{sec:mb-dist}
%We analyse
% start by an analysis of 
%the number of samples needed to
%approximate a discrete distribution with high probability.
%
To learn an unknown discrete distribution $p$ defined on a finite
domain ${\sf Dom}(p)$, we collect i.i.d. samples from that
distribution and infer a model of it. Formally, given a sequence $\seq=(s_j)_{j \in J}$ of samples
drawn i.i.d. from the distribution $p$, we denote by
$p(\seq) : {\sf Dom}(p) \rightarrow [0,1]$, the function that maps
every element $a \in {\sf Dom}(p)$ to its relative frequency in
$\seq$.
%, i.e. to $\frac{|\{ j \in J \mid s_j=a \}|}{|\seq|}$.
The following lemma tells us that
% Using Hoeffding's inequality, it is easy to prove the following.
if the size of $S$ is large enough then the model $p(\seq)$ is close to the actual $p$ with high probability.

\begin{lemma} \label{lem:numsample} For all finite discrete
  distributions $p$ with $|{\sf Dom}(p)| = r$, for all
  $\epsilon, \gamma \in (0,1)$ such that $\pi^p_{\min} > \epsilon$,
  if $\seq$ is a sequence of at least
  $r \cdot \lceil \frac{1}{2\epsilon^2}(\ln 2r-\ln \gamma) \rceil$
  i.i.d. samples drawn from $p$, then $p \sim^\epsilon p(\seq)$ with
  probability at least $1 - \gamma$. 
\end{lemma}
\begin{proof}
% Corresponding to a soft task $\tau_i$ for $i \in F$, 
For a distribution $p$, and an element $e$ in $\dom(p)$, let $X_1^{p_e}, \dots, X_m^{p_e}$ be independent and identically distributed Bernoulli random variables with $\rvexpect{X_j^{p_e}}=\mu$ for $j \in [m]$.
Recall that a Bernoulli random variable takes two values, $1$ and $0$.
In our case, the value $1$ denotes witnessing the element $e$ in the domain of the distribution $p$.
% that is, 
% the success probability of witnessing the support $e$ of a distribution $p$ of a soft task is $\mu$, or stated otherwise, 
Thus we have $p(e)=\mu$.
Let $\overline{X}_m^{p_e} = \frac{1}{m}\displaystyle{\sum_{j \in [m]} X_j^{p_e}}$.
Here $m$ is the number of samples required to learn the probability of occurrence of the element $e$ of the support of the distribution.
% We recall from above that corresponding to the execution time distribution, a sample corresponds to the execution time for completing a job of a task, and for learning an inter-arrival time distribution, a sample corresponds to the time difference between the arrival of two consecutive tasks of a job.

By Hoeffding's two sided inequality, for the special case of Bernoulli random variables, we have,
\[
\rvprob(|\overline{X}_m^{p_e} - \mu| \ge \epsilon) \leq 2\exp(-2m\epsilon^2).
\]
Now we want that the probability of $|\overline{X}_m^{p_e} - \mu| \ge \epsilon$ for all $e \in \dom(p)$ is at most $\frac{\gamma}{r}$, so that the probability of $|\overline{X}_m^{p_e} - \mu| \ge \epsilon$ for some element $e$ in the domain of the distribution $p$ is at most $\gamma$.

Thus we have $2\exp(-2m \epsilon^2) \le \frac{\gamma}{r}$ leading to $m \ge \lceil \frac{1}{2\epsilon^2}(\ln 2r-\ln \gamma) \rceil$.
Since there are $r$ elements in the domain, we need a total of at least $\fn(r, \epsilon, \gamma)=m \cdot r$ samples, and hence the result.
\qed
\end{proof}

We say that we ``PAC learn'' a distribution $p$ if for all
$\epsilon, \gamma \in (0,1)$ such that $\pi^p_{\min} > \epsilon$, by
drawing a sequence $\seq$ of i.i.d. samples from~$p$, we have
$p \sim^\epsilon p(\seq)$ with probability at least $1 - \gamma$.
%Informally, we also refer to this as learning a distribution with
%strong guarantees.
Given a task system $\Upsilon$, if we
can learn the distributions corresponding to all the tasks in
$\Upsilon$, and hence a model $\Upsilon^M$, such
that each learnt distribution in $\Upsilon^M$ is structurally
identical to its corresponding distribution in $\Upsilon$, the
corresponding MDP are structurally identical.

%\begin{proposition}
%  Given two task systems $\Upsilon$ and $\Upsilon'$ with equal number
%  of tasks, if the distributions of all the tasks are structurally
%  identical, then $\Gamma_\Upsilon$ is structurally identical to
%  $\Gamma_\Upsilon'$.
%\end{proposition}

\paragraph{Efficient PAC learning}
\label{sec:mb-onlysoft}
Let $\Upsilon=((\tau_i)_{i \in I},F,\emptyset)$ be a task system with
soft tasks only, and let $\epsilon, \gamma \in (0,1)$.  We assume that
for all distributions $p$ occurring in the models of the tasks in
$\Upsilon$: $\pi^p_{\min} > \epsilon$.  To learn a model
$\Upsilon^M$ which is $\epsilon$-close to $\Upsilon$ with probability
at least $1 - \gamma$, we apply Lemma~\ref{lem:numsample} in the
following algorithm:
\begin{enumerate}
\item for all tasks $i=1,2,\dots \in F$, repeat the following learning
  phase:\\
  Always schedule task $\tau_i$ when a job of this task is
  active. Collect the samples $\seq(\mathcal{A}_i)$
  of ${\cal A}_i$ and $\seq(\mathcal{C}_i)$ of ${\cal C}_i$ as
  observed. Collect enough samples to apply Lemma~\ref{lem:numsample}
  and obtain the desired accuracy as fixed by $\epsilon$ and $\gamma$.  
\item the models of inter-arrival time distribution and computation
  time distribution for task $\tau_i$ are $p(\seq(\mathcal{A}_i))$
  and $p(\seq(\mathcal{C}_i))$ respectively.
\end{enumerate}

\noindent It follows that task systems with only soft tasks are \emph{efficiently} PAC learnable:
  
\begin{theorem} \label{thm:onlysoft}
There is a learning algorithm such that for all task systems $\Upsilon=((\tau_i)_{i \in I},F,H)$ with $H=\emptyset$, for all $\epsilon, \gamma \in (0,1)$, the algorithm learns a model $\Upsilon^M$ such that $\Upsilon^M \approx^{\epsilon} \Upsilon$ with probability at least $1-\gamma$ after executing $\Upsilon$ for $|F| \cdot \mathcal{A}_{\max} \cdot \mathbb{D} \cdot \lceil \frac{1}{2\epsilon^2}(\ln 4\mathbb{D}|F|-\ln \gamma)\rceil$ steps.
\end{theorem}
\begin{proof}
Using Lemma \ref{lem:numsample}, given $\epsilon, \gamma' \in (0,1)$, for every distribution $p$ of the task system, a sequence $\seq$ of $\mathbb{D} \cdot \lceil \frac{1}{2\epsilon^2}(\ln 2\mathbb{D}-\ln \gamma') \rceil$ i.i.d. samples suffices to have $p(\seq) \sim^{\epsilon} p$ with probability at least $1-\gamma'$.
Since in the task system $\Upsilon$, there are $2|F|$ distributions, with probability at least $1-2|F|\gamma'$, we have that the learnt model $\Upsilon^M \approx^{\epsilon} \Upsilon$.
Thus for $\gamma'= \frac{\gamma}{2|F|}$, and using $2\exp(-2m \epsilon^2) \le \frac{\gamma}{2|F|\mathbb{D}}$, we have that for each distribution, a sequence of $\mathbb{D} \cdot \lceil \frac{1}{2\epsilon^2}(\ln 4\mathbb{D}|F|-\ln \gamma) \rceil$ samples suffices so that $\Upsilon^M \approx^{\epsilon} \Upsilon$ with probability at least $1-\gamma$.

Since samples for computation time distribution and inter-arrival time distribution for each soft task can be collected simultaneously, and observing each sample takes a maximum of $\mathcal{A}_{\max}$ time steps, and we collect samples for each soft task by scheduling one soft task after another, the result follows.
\qed
\end{proof}

\paragraph{{\bf Safe learning with hard tasks}}
We turn to task
systems $\Upsilon=((\tau_i)_{i \in I},F,H)$ with both hard and soft
tasks. The learning
algorithm must ensure that all the jobs of hard tasks meet their
deadlines while learning the task distributions. The soft-task-only algorithm
%algorithm 
%proposed above 
%for systems with only soft tasks is
is clearly
not valid for that more general case.
Recall we have assumed schedulability of the task system
%Recall that we assume that the task system
%$\Upsilon=((\tau_i)_{i \in I},F,H)$ for which we want to safely learn
%a model $\Upsilon$ is schedulable
for the hard tasks\footnote{Note that safety synthesis already identifies task systems that violate this condition.}. This is a
necessary condition for safe learning but it is not a
sufficient condition. Indeed, to apply Lemma~\ref{lem:numsample}, we
need enough samples for all tasks $i \in H \cup F$.

First, we note that when executing any safe schedule for the hard
tasks, we will observe enough samples for the hard tasks. Indeed,
under a safe schedule for the hard tasks, any job of a hard task that
enters the system will be executed to completion before its deadline. We then observe the value of the
inter-arrival and computation times for all the jobs of hard tasks that enter the
system. Unfortunately, this is not necessarily the case for soft
tasks when they execute in the presence of hard tasks.
Indeed, it is in general not possible to schedule all the jobs
of soft tasks up to completion. %hence we cannot easily obtain samples
%for the soft tasks.
%\track{As an example, consider a task system such that a job of a soft task with a computation time of $2$ and deadline $4$ arrives at a time when there already exists an active job of a hard task with remaining $3$ units of computation time and a remaining time of $4$ before the deadline.
%The job of the hard task needs to complete execution before its deadline and thus the job of the soft task cannot execute to completion before its deadline.
%Thus, considering as samples only those
%jobs that have been fully scheduled to %completion and ignoring those
%hat do not, would lead to samples that are {\em biased} towards smaller computation times,} and would not allow us to draw conclusions about the real
%computation time distribution.
%
We thus need stronger conditions 
%than hard task schedulability to safely learn task systems. 
in order to be able to learn the distributions of the soft tasks while ensuring safety.
%We develop two such conditions in the rest
%of this section. The first one ensures PAC learnability but does not
%guarantee efficient learning; the
%second one is stricter but ensures efficient PAC learnability.

\paragraph{{\bf PAC guarantees for safe learning}}
Our condition to ensure safe PAC learnability relies on properties of the safe region $\Gamma^{\text{safe}}_\Upsilon$ in the MDP $\Gamma_\Upsilon$ associated to the task system $\Upsilon$. First, note that $\Gamma^{\text{safe}}_\Upsilon$ is guaranteed to be non-empty as the task system $\Upsilon$ is guaranteed to be schedulable for its hard tasks by hypothesis. Our condition will exploit the following property of its structure:

\begin{lemma}\label{lem:singleMEC}
  Let $\Upsilon=((\tau_i)_{i \in I},F,H)$ be a task system and let
  $\Gamma^{\text{safe}}_\Upsilon$
%   =\zug{V,E,L,(V_\Box, V_\ocircle),A, \delta, \cst}$ 
    be the safe region of its MDP. 
    % Then, $(v,E)$ is a MEC in $\Gamma^{\text{safe}}_\Upsilon$.
    Then $\Gamma^{\text{safe}}_\Upsilon$ is a single maximal end-component (MEC).
    % in the case where $I_i=0$ for all tasks $i\in [n]$.
\end{lemma}
\begin{proof}
  We first assume that the task system $\Upsilon=((\tau_i)_{i \in
    I},F,H)$ is schedulable. Otherwise,
  $\Gamma^{\text{safe}}_\Upsilon$ is empty and the Lemma is trivially
  true.  Let $V$ and
  $E$ be the set of vertices and the set of edges of
  $\Gamma^{\text{safe}}_\Upsilon$ respectively.  First, observe that,
  since we want to prove that the whole MDP
  $\Gamma^{\text{safe}}_\Upsilon$ corresponds to an MEC, we only need
  to show that its underlying graph
  $(V,E)$ is strongly connected. Indeed, since
  $(V,E)$ contains all vertices and edges from
  $\Gamma^{\text{safe}}_\Upsilon$, it is necessarily maximal, and all
  choices of actions from any vertex will always lead to a vertex in
  $V$.

  In order to show the strongly connected property, we fix a vertex
  $v\in
  V$, and show that there exists a path in
  $\Gamma^{\text{safe}}_\Upsilon$ from $v$ to
  $\initV$. Since all vertices in
  $V$ are, by construction of
  $\Gamma^{\text{safe}}_\Upsilon$, reachable from the initial vertex
  $\initV$, this entails that all vertices
  $v'$ are also reachable from
  $v$, hence, the graph is strongly connected.

  Let us first assume that $v\in V_\playerOne$, i.e., $v$ is a vertex
  where Scheduler has to take a decision. Let
  $\initV=v_0,v_0',v_1,v_1',\cdots,v_{n-1}',v_n=v$ be the path $\pi$
  leading to $v$, where all vertices $v_j$ belong to
  Scheduler, and all $v_j'$ are are vertices that belong to
  TaskGen. 

  Then, from path $\pi$, we extract, for all tasks $\tau_i$ the
  sequence of \emph{actual inter-arrival times}
  $\sigma_i=t^i(1),t^i(2),\ldots,t^i(k_i)$ defined as follows: for all
  $1\leq j\leq k_i$, $t^i(j)\in \supp(\A_i)$ is the time elapsed (in
  CPU ticks) between the arrival of the $j-1$th job the $j$th job of
  task $i$ along $\pi$ (assuming the initial release occurring in the
  initial state $\initV$ is the $0$-th release). In other words,
  letting $T^i(j)=\sum_{k=1}^jt^i(k)$, the $j$th job of $\tau_i$ is
  released along $\pi$ on the transition between $v_{T^i(j-1)}'$ and
  $v_{T^i(j)}$. Observe thus that all tasks $i\in [n]$ are in the same
  state in vertex $\initV$ and in vertex $v_{T^i(j)}$, i.e. the time
  to the deadline, and the probability distributions on the next
  arrival and computation times are the same in $\initV$ and
  $v_{T^i(j)}$. However, the vertices $v_{T^i(j)}$ can be different
  for all the different tasks, since they depend on the sequence of
  job releases of $\tau_i$ along $\pi$. Nevertheless, we claim that
  $\pi$ can be extended, by repeating the sequence of arrivals of all
  the tasks along $\pi$, in order to reach a vertex where all tasks
  have just submitted a job (i.e. $\initV$). To this aim, we first
  extend, for all tasks $i\in[i]$, $\sigma_i$ into
  $\sigma_i'=\sigma_i, t^i(k_i+1)$, where $t^i(k_i+1) \in \supp(\A_i)$
  ensures that the $k_i+1$ arrival of a $\tau_i$ occurs \emph{after}
  $v$.

  For all $i\in [n]$, let $\Delta_i$ denote
  $\sum_{j=1}^{k_i+1}t^i(j)$, i.e. $\Delta_i$ is the total number of
  CPU ticks needed to reach the first state after $v$ where task $i$
  has just submitted a job (following the sequence of arrival
  $\sigma_i'$ defined above). Further, let
  $\Delta=\textrm{lcm}(\Delta_i)_{i \in [n]}$. Now, let $\pi'$ be a
  path in $\Gamma^{\text{safe}}_\Upsilon$ that respects the following
  properties:
  \begin{enumerate}
  \item $\pi$ is a prefix of $\pi'$;
  \item $\pi'$ has a length of $\Delta$ CPU ticks;
  \item $\pi'$ ends in a $\playerOne$ vertex $v'$; and
  \item for all tasks $i\in [n]$: $\tau_i$ submits a job at time $t$
    along $\pi'$ iff it submits a job at time $t\mod \Delta_i$ along
    $\pi$.
  \end{enumerate}
  Observe that, in the definition of $\pi'$, we do not constrain the
  decisions of Scheduler after the prefix $\pi$. First, let us
  explain why such a path exists. Observe that the sequence of task
  arrival times is legal, since it consists, for all tasks $i$, in
  repeating $\Delta/\Delta_i$ times the sequence $\sigma_i'$ of
  inter-arrival times which is legal since it is extracted from path
  $\pi$ (remember that nothing that Scheduler player does can restrict
  the times at which TaskGen introduces new jobs in the system). Then, since
  $\Upsilon$ is schedulable, we have the guarantee that all $\playerOne$
  vertices in $\Gamma^{\text{safe}}_\Upsilon$ have at least one
  outgoing edge. This is sufficient to ensure that $\pi'$ indeed
  exists. Finally, we observe $\pi'$ visits $v$ (since $\pi$ is a
  prefix of $\pi'$), and that the last vertex $v'$ of $\pi'$ is a
  $\playerOne$ vertex obtained just after \emph{all tasks} have submitted
  a job, by construction. Thus $v'=\initV$, and we conclude that, from
  all $v\in V_\playerOne$ which is reachable from $\initV$, one can
  find a path in $\Gamma^{\text{safe}}_\Upsilon$ that leads back to $\initV$.
 
  This reasoning can be extended to account for the nodes
  $v\in V_\playerTwo$: one can simply select any successor
  $\overline{v}\in V_\playerOne$ of $v$, and apply the above reasoning
  from $\overline{v}$ to find a path going back to $\initV$.\qed
\end{proof}

\paragraph{{\bf Good for sampling}}
The safe region
$\Gamma^{\text{safe}}_\Upsilon$ of the task system
$\Upsilon=((\tau_i)_{i \in
  I},F,H)$ is {\em good for sampling} if for all soft tasks $i \in
F$, there exists a vertex $v_i \in
\Gamma^{\text{safe}}_\Upsilon$ such that:
\begin{inparaenum}[(i)]
\item a new job of task $i$ enters the system in $v_i$; and
\item there exists a strategy
  $\sigma_i$ of Scheduler that is compatible with the set of safe
  schedules for the hard tasks so that from
  $v_i$, under schedule
  $\sigma_i$, the new job associated to task
  $\tau_i$ is guaranteed to reach completion before its deadline.
\end{inparaenum}

There is an algorithm that executes in polynomial time in the size of
$\Gamma^{\text{safe}}_\Upsilon$ and which decides if
$\Gamma^{\text{safe}}_\Upsilon$ is good for sampling. Also, remember
that only the knowledge of the structure of the task system is needed
to compute $\Gamma^{\text{safe}}_\Upsilon$.

Given a task system $\Gamma^{\text{safe}}_\Upsilon$ that is {\em good
  for sampling}, given any $\epsilon,\gamma \in (0,1)$, we safely
learn a model $\Upsilon^M$ which is $\epsilon$-close to $\Upsilon$
with probability at least $1-\gamma$ (PAC guarantees) by applying
the following algorithm:
\begin{enumerate}
\item Choose any safe strategy $\sigma_H$ for the hard tasks, and
  apply it until enough samples
  $(\seq(\mathcal{A}_i),\seq(\mathcal{C}_i))$ for each $i \in H$ have
  been collected according to Lemma~\ref{lem:numsample}. The models
  for tasks $i \in H$ are $p(\seq(\mathcal{A}_i))$ and
  $p(\seq(\mathcal{C}_i))$.
\item Then for each $i \in F$, apply the following phases:
  \begin{enumerate}
  \item from the current vertex $v$, schedule some task uniformly at random among
    the set of tasks that correspond to the safe edges in $\safe(v)$ up to reaching some $v_i$
    %note that
    %there can be more than one such $v_i$ vertices; 
    (while choosing tasks that do not violate safety
    uniformly at random, we reach some $v_i$ with probability $1$.\footnote{This follows from the fact that there is a single MEC in the MDP by Lemma \ref{lem:singleMEC}.}
    %, playing a strategy uniformly at random allows us to reach a state from every other state with probability 1. This is a classical property of random walks in Markov Chains.}}.
    The existence of a $v_i$ is guaranteed by the hypothesis that
    $\Gamma^{\text{safe}}_\Upsilon$ is good for sampling).
  \item from $v_i$, apply the schedule $\sigma_i$ as defined by the second condition in the
    {\em good for sampling condition}. This way we are guaranteed to
    observe the computation time requested by the new job of task $i$
    that entered the system in vertex $v_i$, no matter how TaskGen
    behaves. At the completion of this job of task $i$, we have
    collected a valid sample of task $i$.
  \item go back to $(a)$ until enough samples
    $(\seq(\mathcal{A}_i),\seq(\mathcal{C}_i))$ have been collected
    for soft task $i$ according to Lemma~\ref{lem:numsample}.  %The model
    %for task $i$ is given by $p(\seq(\mathcal{A}_i))$ and
    %$p(\seq(\mathcal{C}_i))$.
  \end{enumerate}
\end{enumerate}
The properties of the learning algorithm above are used to prove that:

% \begin{theorem} \label{thm:good_sampling} There is an
%   algorithm s.t. for all task systems
%   $\Upsilon=((\tau_i)_{i \in I},F,H)$ with a safe region
%   $\Gamma^{\text{safe}}_\Upsilon$ that is good for sampling, for all
%   $\epsilon, \gamma \in (0,1)$, the algorithm learns a model
%   $\Upsilon^M$ such that $\Upsilon^M \approx^{\epsilon} \Upsilon$ with
%   probability at least $1-\gamma$.
% \end{theorem}
\begin{theorem} \label{thm:good_sampling}
There is a learning algorithm such that for all task systems $\Upsilon=((\tau_i)_{i \in I},F,H)$ with a safe region $\Gamma^{\text{safe}}_\Upsilon$ that is good for sampling, for all $\epsilon, \gamma \in (0,1)$, the algorithm learns a model $\Upsilon^M$ such that $\Upsilon^M \approx^{\epsilon} \Upsilon$ with probability at least $1-\gamma$.
\end{theorem}
\begin{proof}
For the hard tasks, as mentioned above, we can learn the distributions by applying the safe strategy $\sigma_H$ to collect enough samples $(\seq(\mathcal{A}_i),\seq(\mathcal{C}_i))$ for each $i \in H$.

We assume an order on the set of soft tasks.
First for all $\tau_i$ for $i \in F$, since $\Gamma^{\text{safe}}_\Upsilon$ is good for sampling, we note that the set $V_i$ of vertices $v_i$ (as defined in the definition of good for sampling condition) is non-empty.
% consider that $\tau_{\Box i} \neq \emptyset$.
Recall from Lemma \ref{lem:singleMEC} that $\Gamma^{\text{safe}}_\Upsilon$ has a single MEC.
% Thus from every vertex of the MEC, there exists a deterministic memoryless strategy of Scheduler to reach $\top_{\Box i}$ almost surely.
% Thus Scheduler has a deterministic memoryless strategy to visit $\top_{\Box i}$ infinitely often almost surely.
Thus from every vertex of $\Gamma^{\text{safe}}_\Upsilon$, Scheduler by playing uniformly at random reaches some $v_i \in V_i$ with probability $1$, and hence can visit the vertices of $V_i$ infinitely often with probability $1$.
Now given $\epsilon$ and $\gamma$, using Theorem \ref{thm:onlysoft}, we can compute an $m$, the number of samples corresponding to each distribution required for safe PAC learning of the task system.
% Also since we consider for learning only jobs that arrive at $v_i$ as \emph{valid} samples, and since all the jobs that correspond to these valid samples finish, we observe that none of the elements in the support is biased towards completing the execution, and we can learn the distribution as desired.
Since by playing uniformly at random, Scheduler has a strategy to visit the vertices of $V_i$ infinitely often with probability $1$, it is thus possible to visit these vertices at least $m$ times with arbitrarily high probability.

Also after we safely PAC learn the distributions for task $\tau_i$,
% \sgcomment{Strictly speaking, we define PAC learning a task system, but not a distribution}
since there is a single MEC in $\Gamma^{\text{safe}}_\Upsilon$, there exists a uniform memoryless strategy to visit a vertex
% $\top_{\Box i+1}$ almost surely.
$v_{i+1}$ corresponding to task $\tau_{i+1}$ with probability $1$.
Hence the result.
% Thus we can learn the distributions for all soft tasks $\tau_i$ arbitrarily closely with arbitrary high probability
% with strong guarantees if $\tau_{\Box i} \neq \emptyset$ for all $i \in F$.
% under the condition that $\Gamma^{\text{safe}}_\Upsilon$ is good for sampling.
\qed
\end{proof}

In the algorithm above, to obtain one sample of a soft task, we need
to reach a particular vertex $v_i$ from which we can safely schedule a
new job for the task $i$ up to completion. As the underlying MDP
$\Gamma^{\text{safe}}_\Upsilon$ can be large (exponential in the
description of the task system), we cannot bound by a polynomial the
time needed to get the next sample in the learning algorithm. So, this
algorithm does not guarantee efficient PAC learning. We develop in the
next paragraph a stronger condition to guarantee efficient PAC learning.

\paragraph{{\bf Good for efficient sampling}}
The safe region $\Gamma^{\text{safe}}_\Upsilon$ of the task system
$\Upsilon=((\tau_i)_{i \in I},F,H)$ is {\em good for efficient
  sampling} if there exists $K \in \mathbb{N}$ which is bounded
polynomially in the size of
$\Upsilon=((\tau_i)_{i \in I},F,H)$, and if, for all soft tasks
$i \in F$ the two following conditions hold:
\begin{enumerate}
\item let $V^{\text{safe}}_\Box$ be the set of Scheduler vertices in
  $\Gamma^{\text{safe}}_\Upsilon$.  There is a non-empty subset
  ${\sf Safe}_i \subseteq V^{\text{safe}}_\Box$ of vertices from which
  there is a strategy $\sigma_i$ for Scheduler to schedule
  safely the tasks $H \cup \{i\}$ (i.e. all hard tasks  \emph{and} the task $i$); and
\item for all $v \in V^{\text{safe}}_\Box, i \in F$, there is a uniform memoryless strategy
  $\sigma_{\diamond {\sf Safe}_i}$ s.t.:
  \begin{enumerate}
  \item $\sigma_{\diamond {\sf Safe}_i}$ is compatible with the safe strategies (for the hard tasks) of
    $\Gamma^{\text{safe}}_\Upsilon$;
  \item when $\sigma_{\diamond {\sf Safe}_i}$ is executed from any
    $v \in V^{\text{safe}}_\Box$, then the set ${\sf Safe}_i$ is
    reached within $K$ steps.
    By Lemma \ref{lem:singleMEC}, since $\Gamma^{\text{safe}}_\Upsilon$ has a single MEC, we have that ${\sf Safe}_i$ is reachable from every $v \in V^{\text{safe}}_\Box$.
  \end{enumerate} 
\end{enumerate}
Here again, the condition can be efficiently decided: there is a polynomial-time
algorithm in the size of
$\Gamma^{\text{safe}}_\Upsilon$ that decides if
$\Gamma^{\text{safe}}_\Upsilon$ is good for efficient sampling.
  
Given a task system $\Gamma^{\text{safe}}_\Upsilon$ that is {\em good
  for efficient sampling}, given $\epsilon,\gamma \in (0,1)$, we
safely and efficiently learn a model $\Upsilon^M$ which is
$\epsilon$-close of $\Upsilon$ with probability at least than
$1-\gamma$ (efficient PAC guarantees) by applying:
\begin{enumerate}
\item Choose any safe strategy $\sigma_H$ for the hard tasks, and
  apply this strategy until enough samples
  $(\seq(\mathcal{A}_i),\seq(\mathcal{C}_i))$ for each $i \in H$ have
  been collected according to Lemma~\ref{lem:numsample}. The models
  for tasks $i \in H$ are $p(\seq(\mathcal{A}_i))$ and
  $p(\seq(\mathcal{C}_i))$.% respectively.
\item Then for each $i \in F$, apply the following phase:
  \begin{enumerate}
  \item from the current vertex $v$, play
    $\sigma_{\diamond {\sf Safe}_i}$ to reach the set ${\sf Safe}_i$.
  \item from the current vertex in ${\sf Safe}_i$, apply the schedule
    $\sigma_i$ as defined above. This way we are guaranteed to observe the computation
    time requested by all the jobs of task $i$ that enter the system.
    %
    %no matter how TaskGen behaves. Hence, on completion of each
    %new job of task $i$, we have collected a valid
    %sample of task $i$.
  \item go to $(b)$ until enough samples
    $(\seq(\mathcal{A}_i),\seq(\mathcal{C}_i))$ are collected
    for task $i$ as per Lem.~\ref{lem:numsample}. The models
    for task $i$ are given by $p(\seq(\mathcal{A}_i))$ and
    $p(\seq(\mathcal{C}_i))$.
  \end{enumerate}    
\end{enumerate}

For a task system $\Upsilon$, let $T$ =
$\mathcal{A}_{\max} \cdot \mathbb{D} \cdot \lceil
\frac{1}{2\epsilon^2}(\ln 4\mathbb{D}|\Upsilon|-\ln \gamma)\rceil$.
The properties of the learning algorithm above are used to prove the
following theorem:

\begin{theorem} \label{thm:good_efficient_sampling}
There exists a learning algorithm such that for all task systems $\Upsilon=((\tau_i)_{i \in I},F,H)$ with a safe region $\Gamma^{\text{safe}}_\Upsilon$ that is good for efficient sampling, for all $\epsilon, \gamma \in (0,1)$, the algorithm learns a model $\Upsilon^M$ such that $\Upsilon^M \approx^{\epsilon} \Upsilon$ with probability at least $1-\gamma$ after scheduling $\Upsilon$ for $T+ |F| \cdot (T+K)$ steps.
\end{theorem}
%
%For the interested reader, 
% We provide an example of a task system in Appendix \ref{app:good_sampling_new} that satisfies the good for sampling condition, but not the stronger good for efficient sampling condition.
\begin{proof}
Consider the algorithm described above.
Since $\sigma_H$ is a safe schedule for the hard tasks, we can observe the samples corresponding to the computation time distribution and the inter-arrival time distribution for all the hard tasks simultaneously while scheduling the system.
Following the proof of Theorem \ref{thm:onlysoft}, the samples required to learn the distributions of the hard tasks can be observed in time $T$.

Now consider an order on the set of tasks.
% Let $r = |\dom(p)|$ for a distribution $\dist$.
% % Let $\mathbb{D} = \displaystyle_{\max_{i \in [n]} \max (|\dom(A_i)|, |\dom(C_i)|)}$.
% Recall from the proof of Theorem \ref{thm:onlysoft}, that we need $r \cdot \frac{1}{2\varepsilon^2}(\ln 2r-\ln \gamma)$ samples \sgcomment{Check the value} to PAC learn distribution $p$ so that the learnt model $\Upsilon^M$ is $\varepsilon$-close to the actual system $\Upsilon$.
% Thus for all tasks $\tau$, we need $\mathbb{D} \cdot \frac{1}{2\varepsilon^2}(\ln 2\mathbb{D}-\ln \gamma)$ samples suffice \sgcomment{Check the value} to PAC learn distribution $p$ so that the learnt model $\Upsilon^M$ is $\varepsilon$-close to the actual system $\Upsilon$.
% We can obtain $\mathbb{D} \cdot \frac{1}{2\varepsilon^2}(\ln 2\mathbb{D}-\ln \gamma)$ samples in at most $T$ time steps.
% Since we can observe the samples corresponding to both computation time distribution and inter-arrival time distribution for each task within $T$ time steps, and 
Under the good for efficient sampling condition, again from the proof of Theorem \ref{thm:onlysoft}, 
% for every task $\tau_i$ in $\Upsilon$, 
we need to execute the system for $|F|T$ time steps for collecting samples to PAC learn the computation time distributions and the inter-arrival time distributions for all soft tasks in $F$.
% Further, for every task $\tau_i$ in $\Upsilon$, after learning the distributions for the task, we need at most $K$ steps to reach the set ${\sf Safe}_{i+1}$, corresponding to the next task in order, from the current vertex in the safe region.
% Thus the algorithm learns the model $\Upsilon^M$ in $|\Upsilon| \cdot (T+K)$ time steps.
Further, for every soft task $\tau_i$ with $i \in F$, from a vertex in $V^{\text{safe}}_\Box$, by using the strategy $\sigma_{\diamond {\sf Safe}_i}$, we reach ${\sf Safe}_i$ in at most $K$ steps.
Hence the result.
\qed
\end{proof}

% \section{Example illustrating learnability with good for sampling condition} 

We note that there indeed exist task systems that satisfy the good for sampling condition, but not the stronger good for efficient sampling condition.
\begin{example} \label{example:good_sampling_new}
Consider the following task system with one hard and one soft task that we want to learn.
More specifically, we want to learn the distributions associated to the tasks in the system.
For the hard task, the computation time distribution is Dirac with support $\{2\}$, the relative deadline is $2$, and the inter-arrival time distribution is also Dirac with support $\{4\}$.
For the soft task, the computation time distribution has the support $\{1,2\}$, the relative deadline is $2$, and the inter-arrival time distribution is also Dirac and has the support $\{3\}$.
We assume that the domain of each distribution is the same as its support.

We can see that during the execution of the task system, for every time $t$, Scheduler does not have a safe schedule from $t$ that also ensures that the soft task will never miss a deadline.
This implies that considering the good for efficient sampling condition, we have ${\sf Safe_i}=\emptyset$ for $i \in F$, and hence the good for efficient sampling condition is not satisfied by this task system.
Thus we cannot ensure safe and efficient PAC learning for this task system.

On the other hand, there exists a schedule such that for all the jobs of the soft task that arrive at time $\text{lcm}(4,3)\cdot n + 6=12n+6$ (assuming that the system starts executing at time $0$) for $n \ge 0$ can be scheduled to completion, and thus by Theorem \ref{thm:good_sampling}, there exists an algorithm to safely PAC learn the task system.
\qed
\end{example}

\paragraph{{\bf Using the learnt model}}
Given a system $\Upsilon$ of tasks, and parameters
$\epsilon, \gamma \in (0,1)$, once we have learnt a model $\Upsilon^M$
such that $\Upsilon^M \approx^\epsilon \Upsilon$, we construct the MDP
$\Gamma^{\text{safe}}_{\Upsilon^M}$. From
$\Gamma^{\text{safe}}_{\Upsilon^M}$, we can compute an optimal
scheduling strategy that minimises the expected mean-cost of missing
deadlines of soft tasks. Such an algorithm is given in~\cite{ggr18}.
Then, we execute the actual task system $\Upsilon$ under schedule
$\sigma$. However, since $\sigma$ has been computed using the model
${\Upsilon^M}$, it might not be optimal in the original, unknown taks
system $\Upsilon$.  Nevertheless, we can bound the difference between
the optimal values obtained in
$\Gamma^{\text{safe}}_{\Upsilon^M}$ and $\Gamma^{\text{safe}}_{\Upsilon}$.

The following lemma relates the model that is learnt with the
approximate distribution that we have in the MDP corresponding to the
learnt model.
%for the system of tasks.
Given $\epsilon \in (0,1)$, let $s=\min\{1,\pi_{\max}^\Upsilon+\epsilon\}$ and $\eta = s^{2n}-(s-\epsilon)^{2n}$, where $n=|\Upsilon|$.
% \begin{lemma}[From~\cite{KPR18}] \label{lem:learn-dist} Let $\Upsilon$ be a task system,
\begin{lemma} \label{lem:learn-dist} Let $\Upsilon$ be a task system,
  let $\epsilon, \gamma \in (0,1)$, let $\Upsilon^M$ be the learnt
  model such that $\Upsilon^M \approx^\epsilon \Upsilon$ with
  probability at least $1-\gamma$.  Then we have that
  $\Gamma_{\Upsilon^M} \approx^\eta \Gamma_{\Upsilon}$ with probability
  at least $1-\gamma$.
\end{lemma}
\begin{proof}
  Since we have that $\Upsilon^M \approx^\epsilon \Upsilon$ with
  probability at least $1-\gamma$, by definition, we have that the
  probability that all the distributions of $\Upsilon^M$ are
  $\epsilon$-close to their corresponding distributions in $\Upsilon$
  is at least $1-\gamma$.  Let $|\Upsilon|=n$, and there are a total
  of $2n$ distributions.  Let $\delta$ and $\delta^M$ be the
  distribution assignment functions of $\Gamma$ and $\Gamma^M$
  respectively.  Thus corresponding to $\delta$ in $\calM$, if an edge
  has probability $p=p_1p_2\cdots p_{2n}$, and for $\delta^M$ we have
  the corresponding probability as $p^M$, then
  $\lvert p^M-p \rvert \le 
%   (p_1+\epsilon)(p_2+\epsilon)\dots
%   (p_{2n}+\epsilon) - p_1p_2\dots p_{2n} \le
%   \frac{\epsilon}{1-\epsilon}$, 
  \displaystyle{\prod_{i=1}^{2n}(p'_i) - \prod_{i=1}^{2n}(p_i)}$, where $p'_i$ is the estimation of $p_i$ in $\delta^M$, and is such that $p_i'\le \min\{1, p_i+\epsilon\}$, since each estimated probability in the distribution $\delta^M$ is also bounded above by $1$.
  Now $p_i' \le s$ for all $i \in [2n]$, and we have that $\displaystyle{\prod_{i=1}^{2n}p'_i - \prod_{i=1}^{2n}(p_i)} \le s^{2n} - (s-\epsilon)^{2n}$,
%   $\displaystyle{\Pi_{i=1}^{2n}(p_i+\epsilon) - \Pi_{i=1}^{2n}(p_i)}$ 
  and thus $\delta^M \sim^\eta \delta$ with probability
  at least $1-\gamma$.  \qed
\end{proof}

A strategy $\sigma$ is said to be \emph{(uniformly)
  expectation-optimal} if for all $v \in V_\playerOne$, we have
$\expect_{v}^{\Gamma[\sigma]}(\mpay) = \inf_\tau
\expect_{v}^{\Gamma[\tau]}(\mpay)$. 
%Recall that a good strategy minimizis
%one in which the cost is minimised.
The following Lemma captures the
idea that some expectation-optimal strategies for MDPs whose
transition functions have the same support as that of $\Gamma$ are
`robust'.
%That is, when used to play in another structurally
%identical MDP with the same support, and close transition functions,
%they achieve near-optimal expectation.  We obtain the following Lemma
%which can be derived from ~\cite[Theorem~5]{Chatterjee12}. 
\begin{lemma}[Adapted from~{\cite[Theorem~5]{Chatterjee12}}] \label{lem:robust-opt-lem} Consider
  $\beta \in (0,1)$, and MDPs $\Gamma$ and $\Gamma'$ such that
  $\Gamma \approx^{\eta_\beta} \Gamma'$ with
  \( \eta_\beta \leq \frac{\beta \cdot \pmin}{8|V_\playerOne|},
  \) where $\pmin$ is the minimum probability appearing in $\Gamma$.
  For all memoryless deterministic expectation-optimal strategies
  $\sigma$ in $\Gamma'$, for all $v \in V_\playerOne$, it holds that
  \( \left| \expect_{v}^{\Gamma[\sigma]}(\mpay) - \inf_{\tau}
    \expect_{v}^{\Gamma[\tau]}(\mpay) \right| \leq \beta.  \)
\end{lemma}
\begin{proof}
    Recall that in our case, we have that the cost of missing the deadlines of the soft tasks are known, and thus we have the same cost function $\cst$ in both $\Gamma$ and $\Gamma'$.
    The bounds for $\eta_\beta$
    % and $\max\{|r(q,a,q') - r'(q,a,q')| :
    % (q,a,q') \in \supp{(\delta})\}$ are 
    is obtained directly from Solan's
    inequality~\cite[Theorem 6]{Solan03} as adapted by Chatterjee~\cite[Proposition
    1]{Chatterjee12}:
      \begin{equation}\label{eqn:solan}
        \left|
        % \inf_{\tau_1} \expectnew{\calM^{\tau_1}}{q_0}{MP}
        \inf_{\tau_1}
        \expect_{\initV}^{\Gamma[\tau_1]}(\mpay)
        -
        % \inf_{\tau_2} \expectnew{\calM^{\tau_2}}{q_0}{MP}
        \inf_{\tau_2}
        \expect_{\initV'}^{\Gamma'[\tau_2]}(\mpay)
        \right|
        \leq
        \frac{4|V_\playerOne|(\eta_\beta/\pmin)}{1 - 2|V_\playerOne|(\eta_\beta/\pmin)}
    \end{equation}
    % It follows from the inequalities required of $\eta_\varepsilon$ and
    % $\lVert {r-r'} \rVert_\infty$ in Lemma~\ref{lem:robust-opt}, together with
    % the above
    % inequalities, that the left-hand side of Inequality~\eqref{eqn:solan} is at
    % most $\varepsilon/2$. 
%    Indeed, we have required them to be twice as small as
%    necessary. This is because Chatterjee has shown 
    In the proof
    of~\cite[Theorem 5]{Chatterjee12}), it has been shown that if the optimal expected values of
    two structurally identical MDPs differ by at most $\lambda$, then a
    memoryless expectation-optimal strategy for one MDP is
    $2\lambda$-expectation-optimal for the other one.
    
    Thus $\frac{4|V_\playerOne|(\eta_\beta/\pmin)}{1 - 2|V_\playerOne|(\eta_\beta/\pmin)} \le \frac{\beta}{2}$ that gives us
    $\eta_\beta \le \frac{\beta \cdot \pi_{\min}}{8|V_\playerOne|+2|V_\playerOne|\beta} \le \frac{\beta \cdot \pi_{\min}}{8|V_\playerOne|}$.
    The result thus follows.
    \qed
\end{proof}

One of the results we cite, i.e.~\cite[Theorem 5]{Chatterjee12},
focuses on stochastic parity games with the same support, i.e., for
structurally identical MDPs.
There, they derive robustness bounds for MDPs with the discounted-sum function and use them
to obtain robustness bounds for MDPs with a parity objective.
We are, however, extending those results to MDPs with the mean-cost function
(cf.~\cite{DHKP17}) making use of an observation by Solan~\cite{Solan03}:
robustness bounds for discounted-sum MDPs extend directly to mean-cost MDPs if
they do not depend on the discount factor.

% The proof of the above lemma uses Thm. 6 in \cite{Solan03} and
% Thm. 5 in \cite{Chatterjee12}.
%We say a strategy $\sigma$
%such as the one in the result above is
%\emph{$\beta$-robust-optimal} w.r.t. the optimal expected
%mean cost and we call $\beta$ the \emph{robustness precision}.
%
%Then, u
Finally, using both Lemma~\ref{lem:learn-dist} and
Lemma~\ref{lem:robust-opt-lem}, we obtain the following  guarantees on the quality of the scheduler that our
model-based learning algorithm outputs:
\begin{theorem}\label{thm:robust-opt}
  Suppose we are given a task system $\Upsilon$ (with min
  probability $\pmin$) and a robustness precision 
  $\beta\in
  (0,1)$. Let $\gamma,\epsilon\in (0,1)$ be s.t.
  $\epsilon\leq \frac{\beta \pmin}{8|V_\playerOne|+\beta\pmin}$.
  Let $\Upsilon^M$ be the model that is learnt using the above algorithms
  s.t. $\Upsilon^M \approx^\varepsilon \Upsilon$ with probability at least $1-\gamma$,
%   $\Upsilon^M$ be the model of $\Upsilon$ learnt with parameters
%   $\epsilon$ and $\gamma$ using the above algorithms, 
  and let $\sigma$
  be a memoryless deterministic expectation-optimal strategy of
  $\Gamma_{\Upsilon^M}$. Then, with probability at least $1-\gamma$, the expected
  mean-cost of playing $\sigma$ in $\Gamma_\Upsilon$ (i.e. in the task
  system $\Upsilon$) 
  is s.t. for all
  $v \in V_\playerOne$:
  \( \left| \expect_{v}^{\Gamma_\Upsilon[\sigma]}(\mpay) - \inf_{\tau}
    \expect_{v}^{\Gamma_\Upsilon[\tau]}(\mpay) \right| \leq \beta.  \)
\end{theorem}

% \begin{remark} \label{rem:MB} Note that in the above Theorem, for the
%   mean-cost to be approximately the same, the supports of the
%   distributions assigned by $\delta$ and $\delta'$ need to be the
%   same.  We note that $\eta_{\epsilon} \le \pi_{\min}$.  We can
%   consider $\eta_{\epsilon}$ such that it is much smaller than
%   $\pi_{\min}$ so that if we learn a distribution $\delta'$, and
%   $\delta' \sim^{\eta_{\epsilon}} \delta$ with probability at least
%   $1 - \gamma$, we also have that $\supp{(\delta)}=\supp{(\delta')}$
%   with probability at least $1-\gamma$.
% \end{remark}

%%% Local Variables:
%%% mode: latex
%%% TeX-master: "main"
%%% End:

\section{Monte Carlo Tree Search with Advice}\label{sec:mcts}
When the model of the task system is known, or once it has been learned using techniques developed in Section~\ref{sec:model_based}, our goal is to compute a (near) optimal strategy while ensuring safe scheduling of hard-tasks with certainty. 

The challenge is the sizes of the MDPs that are too large for exact model-checking techniques (see Sect.~\ref{sec:experimental-results}). To overcome this problem, we resort to
%an heuristic approach
%An alternative approach to handle large state spaces 
%consists in using 
a \emph{receding horizon} framework \cite{DBLP:journals/ml/KearnsMN02}, that bases its decisions on a finite-depth unfolding of the MDP from the current state. 
In particular, we advocate the use of \emph{Monte Carlo Tree Search} (MCTS) algorithms~\cite{DBLP:journals/tciaig/BrownePWLCRTPSC12}, that are a popular method for sampling the finite-depth unfolding while avoiding an exponential dependency on the horizon. MCTS algorithms aim at discovering and exploring the ``most relevant" parts of the unfolding, and they approximate the value of actions in intermediary nodes using a fixed number of trajectories obtained by simulations. The MCTS algorithm builds an exploration tree incrementally. At every step of the algorithm, the {\em selection phase} selects a path in the current tree, possibly extending it by adding a new node. It is followed by a {\em simulation phase}, that extends this trajectory further, until the fixed horizon is reached. Finally, a {\em back-propagation} phase updates the exploration tree based on this new trajectory. A reader looking for a more detailed introduction to MCTS is referred to~\cite{DBLP:conf/concur/Busatto-Gaston020}.

MCTS has been successfully applied to large state-spaces. For example, it is an important building block of the {\sc AlphaGo} algorithm~\cite{DBLP:journals/nature/SilverHMGSDSAPL16} that has obtained super-human performances in the game of {\sc Go}. Such level of performances cannot be obtained with the plain MCTS algorithm. In {\sc Go}, the simulation and selection phases are guided by a board scoring function that has been learned using neural-networks techniques and self-play. 
%This guidance is essential in all non-trivial application of the MCTS algorihm.
For our scheduling problem, we also need a solution to this guidance problem and, equally importantly, we must augment the MCTS algorithm in a way that \emph{ensures} safe scheduling of hard tasks. 

\paragraph{\bf Symbolic advice}
In a recent previous work~\cite{DBLP:conf/concur/Busatto-Gaston020}, we have introduced the notion of (symbolic) advice that provides a generic and formal solution to systematically incorporate domain knowledge in the MCTS algorithm. For our scheduling problem, we use selection advice that prunes parts of the MDP on-the-fly in order to ensure that only safe schedulers are explored. We have considered two possibilities. First, we consider the most general safe scheduler (MGS scheduler) as defined in page~\pageref{sec:mgs-def} to restrict the selection phase to safe scheduling decisions only. Second, we consider the earliest deadline first (EDF) scheduling strategy for hard tasks defined in page~\pageref{def:EDF}, that only allows soft tasks when there are no available hard tasks,
and restricts to the hard tasks with the earliest deadline otherwise.
EDF is guaranteed safe as the set of hard tasks is assumed schedulable. 
% The EDF advice restrict decisions to the hard task returned by EDF is there are hard tasks to schedule, and to the set of soft tasks if there are no available hard tasks.\todo{D:added this}
%The advantage of MGS is that it 
The MGS advice allows for maximal exploration as it leaves open all possible safe scheduling solutions, while the EDF advice can be applied on larger task systems as it does not require any precomputations.
% and  is applicable to any instances in which the set of hard tasks is schedulable. 
These advice are also applicable during the simulation phases. 

\section{Experimental Results}\label{sec:experimental-results}

% \iffalse
% \begin{figure}[t]
%     \centering
% \subcaptionbox{System with 4 soft tasks\label{fig:onlysoft4_dist}}
% {\includegraphics[width=1\textwidth]{figs/fig_onlysoft4_dist.pdf}}
% \subcaptionbox{System with 6 soft tasks\label{fig:soft6_dis}}
% {\includegraphics[width=1\textwidth]{figs/fig_soft6_dist.pdf}}
% \caption{$L1$ (sum), $L2$ (Euclidean), $L\infty$ (max) norm distances between the learnt and the actual distributions; ``exe'' and ``arr'' stand for computation time distribution and inter-arrival time distribution respectively.}\label{fig:softies}
% \end{figure}
% \fi

%\todo[inline]{G: Mention pMC tools and the fact that we can use them to
 % compare schedulers via our parametric MDP abstraction}

\begin{figure}[t]
\vspace*{-0.2cm}
\begin{minipage}[b]{0.45\linewidth}
\vspace*{-5mm}
% \includegraphics[width=1\textwidth]{figs/fig_onlysoft4_dist.pdf}
% \caption{\label{fig:onlysoft4_dist} $L1$ (sum), $L2$ (Euclidean), $L\infty$ (max) norm distances between learnt and actual distributions 
% for $4$ soft tasks.}
%; ``exe'' and ``arr'' stand for computation time distribution and inter-arrival time distribution respectively.}
\includegraphics[width=1\textwidth]{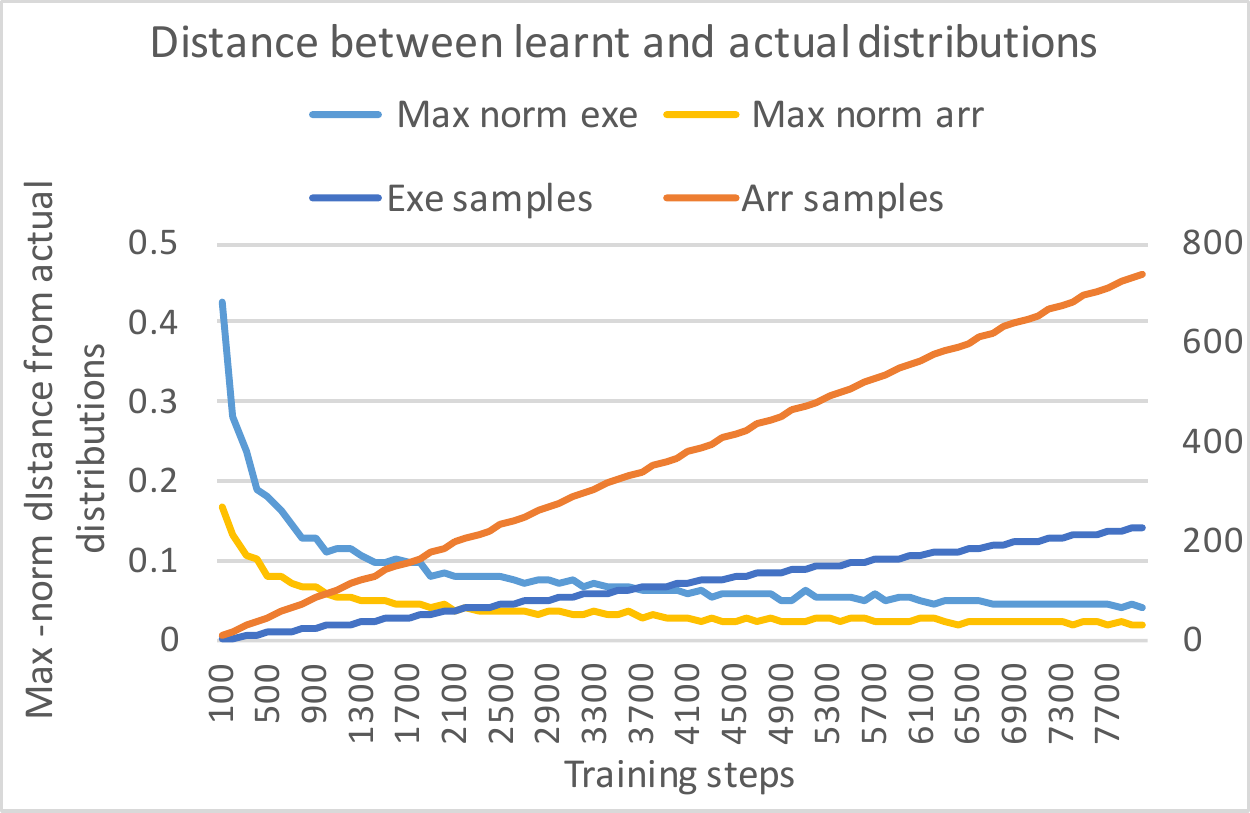}
\caption{\label{fig:soft6_dist}%
Learning distributions for a system with $6$ soft tasks.}
\end{minipage}
\quad
\begin{minipage}[b]{0.48\linewidth}
\centering
\includegraphics[width=1\textwidth]{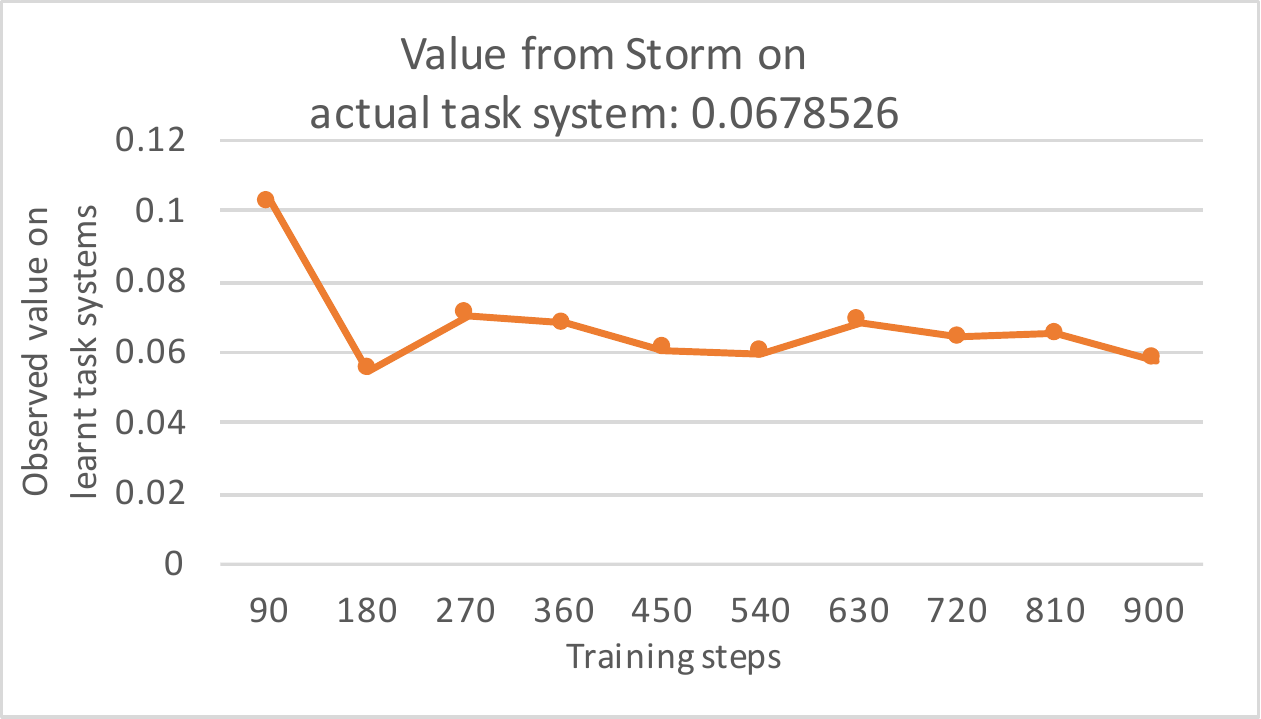}
\caption{\label{fig:soft2MB} Model-based learning for 1 hard, 2 soft tasks}
\end{minipage}
\end{figure}

In this section, we first report experimental results on model-based learning and observe that the models are learnt efficiently with only a small number of samples.
Our MCTS based algorithms can then be applied on the learnt models that are very close to the original ones.\footnote{Here we do not learn to the point where our PAC guarantees hold. Rather, we are interested in how fast the learnt model converges to the real model in practice.}
We compare the performance of our MCTS-based algorithms with a state-of-the-art deep $Q$-learning implementation from {\sc OpenAI} \cite{baselines} on a set of benchmarks of task systems of various sizes.
% First, we use {\sc Storm}~\cite{DJKV17} to compute optimal strategies from known models and report that the observed model-based results corroborate our theoretical developments.
% Then we compare our MCTS-based algorithms with baseline $Q$-learning implementation from {\sc OpenAI} \cite{baselines}.
The experimental results show that our MCTS-based algorithms perform better in practice than safe reinforcement learning (RL)\cite{ABCHKP19}.

% we use the baseline $Q$-learning implementation from {\sc OpenAI}.
%from models of task systems (exact or learnt). 
% We report both on model-based learning and on model-free learning, and comment on how the observed results corroborate our theoretical developments.

\paragraph{\bf Models with only soft tasks}
%Recall that in model-based learning, given the structure of a task system $\Upsilon$, and $\epsilon, \gamma \in (0,1)$, we learn a model $\Upsilon^M$ such that $\Upsilon^M \approx^\epsilon \Upsilon$ with probability at least $1-\gamma$.
%
% We report the $L1$, $L2$, and $L \infty$ norms between the distributions of the actual task system and the model that is learnt. The figures give the evolution of these distances 
% against the number of training steps.
% We consider two different task systems, one with four soft tasks (Fig.~\ref{fig:onlysoft4_dist}), and the other with six soft tasks (Fig.~\ref{fig:soft6_dist})\todo{I don't understand what are the 6 plots? G.}.
% Our experimental results support the theory: with a reasonable number of samples, we are able to learn models of the distributions that are close to the actual distributions. 
In Figure~\ref{fig:soft6_dist}, we show that the distributions of a task system with soft tasks can be learnt efficiently with a small number of samples, corroborating our theory in Section~\ref{sec:model_based}.
% In a general MDP, the valid samples may require reaching some specific states of the MDP, and we may require a lot of time to reach these states, and thereby collect samples.
This is not the case in general for arbitrary MDPs where in order to collect samples, one may need to reach some specific states of the MDP, and it may take a considerable amount of time to reach such states.
However, in this case of systems with only soft tasks, the number of samples increases linearly with time.
As a representative task system, we display the learning curve for a system with six soft tasks in Figure~\ref{fig:soft6_dist}.
Here ``exe" and ``arr" refer to the distributions of the computation times and the inter-arrival times respectively.
The left $y$-axis is the max-norm distance between the probabilities in the actual distributions and the learnt distributions across all soft tasks.
The $x$-axis is the number of time steps over which the system is executed.
For learning the computation time distribution, the soft tasks are scheduled in a round robin manner.
% Assuming $t$ time steps, and $n$ tasks in a system, each soft task is scheduled for $\lfloor t/n \rfloor$ steps. For a soft task, a sample of computation time corresponds to the execution time of a job of the task that arrives the system during these $\lfloor t/n \rfloor$ steps over which the task is always scheduled.
Once a job of a soft task is scheduled, it is executed until completion without being preempted.
A sample for learning the computation time distribution of a soft task thus corresponds to a job of the task that is scheduled to execute until completion.
Since the system has only soft tasks, a job can always be executed to finish its execution without safety being violated.
On the other hand, the samples for learning the inter-arrival time distribution for each task correspond to all the jobs of the task that arrive in the system.
Thus over a time duration, for each task, the number of samples collected for learning the inter-arrival time distribution is larger than the number of samples collected for learning the computation time distribution.
The number of samples of both kinds increases linearly with time.
The $y$-axis on the right corresponds to the number of samples collected over a duration of time when the system executes.
The plot ``Exe samples" corresponds to the number of samples collected per task for learning the computation time distributions.
Since the tasks are executed in a round robin manner, the tasks have an equal number of samples for learning their computation time distributions.
% On the other hand, for learning the inter-arrival time distribution of a task, every job of the task that arrives the system corresponds to a sample.
On the other hand, for learning inter-arrival time distributions, a task with larger inter-arrival time produces fewer samples than a task with smaller inter-arrival time.
The plot ``Arr samples" corresponds to the minimum of the number of jobs, over all the tasks, that arrived in the system.
Each point in the graphs
% the max-norm distance between the actual and the learnt distributions and 
is obtained as a result of averaging over $50$ simulations. 
%We note that the learning algorithm for soft tasks only does not need to construct the graph of the MDP associated to the task system to learn the distributions. 
% The complexity of the learning algorithm only increases linearly with the number of soft tasks and polynomially with the required precision ($\epsilon,\gamma$).

\paragraph{\bf Safe model-based learning}
% Now we consider safe model-based learning of systems with both hard and soft tasks.
% In this benchmark, we consider a task system that has one hard and two soft tasks. The limited size of the task system allows us to compare the value of the strategy obtained on the learnt model and the one obtained on the actual model of the task system using {\sc Storm}.
% First, we compute the most general safe scheduler using AbsSynthe~\cite{abssynthe}, a state-of-the-art reactive synthesis solver~\cite{syntcomp14}. 
For safe model-based learning of systems with both hard and soft tasks, first, we verify that the task system satisfies the {\em good for efficient sampling} condition, and hence admits safe efficient PAC learning. We consider a small representative task system, and report the value of the optimal expected mean-cost strategy as computed by {\sc Storm} on the learnt model as a function of the number of steps for which the system is executed (training steps).
%As we can see in Fig.~\ref{fig:soft2MB} t
This %value 
converges quickly to %wards 
the %expected value of the
optimal 
expected value of the actual task system, roughly equal to $0.06$ (see Fig~\ref{fig:soft2MB}). 
% We also note that the value computed on the learnt model may sometimes be smaller than on the actual model 
We also note that the expected value computed by {\sc Storm} is not necessarily monotonic as it is computed on the learnt model and this model changes over time with the samples that it receives, and the expected value may also sometimes be smaller than the value on the actual model.
% Finally, 
% Again, we observe that even with few training steps, the model-based learning algorithm learns a strategy that gives a value that is is close to the optimal one.
The results show that this approach is effective in terms of the quality of learning and the number of samples required.

\paragraph{\bf MCTS}
In the above approach, the main bottleneck towards scalability is the extraction of an optimal strategy from the learnt model using probabilistic model-checkers like {\sc Storm}.
This is because the underlying MDP grows exponentially with the number of tasks.
% Our experimental results on model-based learning (Figs. \ref{fig:soft6_dist} and \ref{fig:soft2MB}) show that this approach is effective in terms of the quality of learning and the number of samples required.
Therefore we advocate the use of receding horizon techniques instead,
that optimize the cost based on the next $h$ steps for some horizon $h$.
In our examples, the unfoldings have approximately $2^h$ states, so we use MCTS to explore them in a scalable way.
% once the probabilities are learnt, for task systems that satisfy the good for efficient sampling condition\todo{D:our examples satisfy this condition right?}.
% we can apply MCTS on models that are learnt using model-based learning.
% Our results show that this works much better than shielded Q-learning, particularly for large task systems, and with small sample size.
%using probabilistic model-checking. This is explained by the fact that the underlying MDP grows exponentially with the number of tasks. A solution to this problem is to use the model-free approach. 
% We report on experiments for the model-free approach below. 
% \begin{figure}[t]
% \vspace*{-0.2cm}
% \begin{minipage}[b]{0.48\linewidth}
% \centering
% \includegraphics[width=1\textwidth]{figs/fig_soft2_MB.pdf}
% \caption{\label{fig:soft2MB} MB learning for 1 hard, 2 soft tasks}
% \end{minipage}
% \quad
% \begin{minipage}[b]{0.45\linewidth}
% \vspace*{-5mm}
% \includegraphics[width=1\textwidth]{figs/fig_onlysoft4_1_4.pdf}
% \caption{\label{fig:onlysoft4} MB and MF learning for 4 soft tasks}
% \end{minipage}
% \end{figure}

\paragraph{\bf Deep Q-learning}
One of the most successful model-free learning algorithm is the
\emph{Q-learning} algorithm, due to Watkins and Dayan~\cite{wd92}.
It aims at learning
(near) optimal strategies in a (partially unknown) MDP for the {\em discounted sum} objective.
In our scheduling problem, we search for (near) optimal strategies for
the mean-cost and {\em not} for the discounted sum, as we want to
minimise the limit average of the cost of missing deadlines of soft
tasks. 
However, if the discount factor is close to $1$, both values coincide~\cite{Solan03,mn81}.
In our experiments, we use an implementation of deep $Q$-learning available in the \textsc{OpenAI} repository \cite{baselines}.
We make use of %advocate here the use of
shielding~\cite{CNPRZ17,shields,ABCHKP19}, a technique that restricts actions in the learning process so that only those actions that are safe for the hard tasks can be used.

\paragraph{\bf Experimental setup for MCTS and deep Q-learning}
We compare some variants of model-based learning augmented with MCTS and some variants of deep Q-learning in the context of scheduling.
The first option is to set a very high penalty on missing the deadline of a hard task, and then to apply either MCTS or deep $Q$-learning. However, safety is not guaranteed in this case, and we report on whether a violation was observed or not. We call this variant unsafe MCTS and unsafe deep $Q$-learning respectively as a consequence.
The second option is to enforce safety in MCTS and deep $Q$-learning by computing the most general safe  scheduler for hard tasks, and then using the MGS advice for MCTS or the MGS shield for deep $Q$-learning.
The third option is to use the earliest-deadline-first (EDF) scheme on hard tasks instead of MGS as an advice or a shield.
Note that the second and the third options are required to ensure safety, and thus are applicable to systems that have at least one hard task, and hence are not applicable (NA) to systems with only soft tasks. 
% The variants of model-free deep Q-learning are the following:
% Both kinds of learning have the following variants:
% \begin{inparaenum}[(i)]
%    \item {\em MGS shielded:} compute the most general safe (MGS) scheduler for hard tasks, for MCTS, we apply MGS for both selection advice as well as for simulation advice, while for deep Q-learning, we restrict actions to those allowed by the MGS.
%    \item {\em EDF shielded:} for MCTS, we apply EDF for both selection as well as for simulation advice and for deep-Q learning, the shield restricts an action to an active hard task with the earliest deadline.
%    \item {\em High cost on missing deadlines of hard tasks:} safety is not guaranteed here, but we apply a high penalty on missing the deadline of a hard task for both MCTS as well as for deep-Q learning.
%\end{inparaenum}

\begin{table}[t]
\vspace*{-0.2cm}
\centering
\scalebox{0.95}{
{\def\arraystretch{1.1}
\begin{tabular}{|l|l|l|l|l|l|l|l|l|}
\hline
Task                                                    & size      & \begin{tabular}[c]{@{}l@{}}Storm\\ output\end{tabular} & \begin{tabular}[c]{@{}l@{}}MCTS \\ unsafe\end{tabular} & \begin{tabular}[c]{@{}l@{}}MCTS\\ MGS\end{tabular} & \begin{tabular}[c]{@{}l@{}}MCTS\\ EDF\end{tabular} & \begin{tabular}[c]{@{}l@{}}Deep-Q\\ unsafe\end{tabular} & \begin{tabular}[c]{@{}l@{}}Deep-Q\\ MGS\end{tabular} & \begin{tabular}[c]{@{}l@{}}Deep-Q\\ EDF\end{tabular} \\ \hline
4S                                                      & $10^5$    & 0.38                                                   & 0.52                                                   & NA                                                 & NA                                                 & 0.56                                                    & NA                                                   & NA                                                   \\ \hline
5S                                                      & $10^6$    & T.O.                                                   & 0                                                      & NA                                                 & NA                                                 & 0.13                                                    & NA                                                   & NA                                                   \\ \hline
10S                                                     & $10^{18}$ & T.O.                                                   & 0                                                      & NA                                                 & NA                                                 & 0.96                                                    & NA                                                   & NA                                                   \\ \hline
simple
% \begin{tabular}[c]{@{}l@{}}1H, 2S\\ simple\end{tabular} 
& $10^2$    & 0                                                      & 0.72                                                    & 0                                                  & 0                                                  & 1.08                                                    & 0.1                                                  & 0                                                    \\ \hline
1H, 2S                                                  & $10^4$    & 0.07                                                   & 0.67                                                   & 0.14                                               & 0.28                                              & 0.24                                                    & 0.11                                                 & 0.22                                                 \\ \hline
1H, 3S                                                  & $10^5$    & 0.28                                                   & 1.13                                                   & 0.45                                              & 0.49                                              & $\infty$ % $2.26^*$
& 0.47                                                 & 0.47                                                 \\ \hline
2H, 1S                                                  & $10^4$    & 0                                                      & 0.92                                                   & 0                                                  & 0.2                                               & $\infty$ % $46.85^*$
& 0.02                                                 & 0.3                                                  \\ \hline
2H, 5S                                                  & $10^{10}$ & T.O.                                                   & 3.44                                                   & 1.93                                             & 2.14                                               & $\infty$ % $176.5^*$
& 2.39                                                 & 2.48                                                 \\ \hline
3H, 6S                                                  & $10^{14}$ & T.O.                                                   & 4.17
& 2.88                                               & 2.97                                               & $\infty$ % $163.01^*$
& 3.42                                                 & 3.47                                                 \\ \hline
2H, 10S                                                 & $10^{22}$ & T.O.                                                   & 0.3                                                      & 0.03                                               & 0.03                                               & $\infty$ % $365.34^*$
& 1.42                                                 & 1.6                                                  \\ \hline
4H, 12S                                                 & $10^{30}$ & T.O.                                                   & 2.1
& 1.2                                               & 1.3                                              & $\infty$ % $996.19^*$
& 2.68                                                 & 2.87                                                 \\ \hline
\end{tabular}
}
}
\caption{\label{tab:results} Comparison of MCTS and reinforcement learning.}
\end{table}

\paragraph{\bf Experimental Results} In the first column of Table \ref{tab:results}, we describe the task systems that we consider.
A description \textsf{2H, 5S} refers to a task system with two hard tasks and five soft tasks, while \textsf{4S} refers to a task system with four soft tasks and no hard tasks.
The \textsf{simple} system refers to a \textsf{1H, 2S} task system where all the arrival time distributions are Dirac.
% The size of the state space for each task system has been mentioned in thousands in the second column.
% For all tasks, with at most four tasks, the exact size of the state space and a value corresponding to an optimal strategy is computed by {\sc Storm}.
The output of {\sc Storm} for the smaller task systems is given in the third column.
%{\sc Storm} times out on the larger benchmarks.
We report sizes of the MDPs, computed with {\sc Storm} whenever possible.
Otherwise we report an approximation of the size of the state space obtained by
taking the product of $(c_i+1)(a_i+1)$ over the set of tasks, where $c_i$ and $a_i$ are the greatest elements in the support of the distributions $\mathcal C_i$ and $\mathcal A_i$.
%multiplying parameters otherwise.\footnote{We take the product of $(c_i+1)(a_i+1)$ over the set of tasks, where $c_i$ and $a_i$ are the greatest elements in the support of the distributions $\mathcal C_i$ and $\mathcal A_i$.}
% For these larger task systems, we compute an approximate size of the state space.
% \todo{One simple method could be multiply by 20 for each addition of a task.}
Recall that the size of the state space is exponential in the number of tasks in the system.
% The fourth column of the table refers to the size of the unfolding of the MDP considered for MCTS.
% In our experiments, the receding horizon length is set to $30$.\todo{Say something about the basis of the exponents.}
% The first column for each of MCTS and deep Q-learning refers to whether a deadline miss for the hard tasks is observed when no guarantee on safety is made, but a high cost is assigned for missing such deadlines.
In the columns where safety is not guaranteed, $\infty$ denotes an observed violation (a missed deadline for a hard task).

For MCTS, at every step we explore 500 nodes of the unfolding of horizon 30, and the value of each node is initialized using 100 uniform simulations.
This computation takes 1-4 minutes in our Python implementation for different benchmarks, running on a standard laptop. It is reasonable to believe that a substantial speedup could be obtained with well-optimised code and parallelism.
%A final weight is attached to leaves and is added to the cost of missed soft tasks along the trajectory\todo{D: not sure what it was}.
For deep $Q$-learning, we train each task system for $10000$ steps.
The implementation of deep-Q learning in the \textsc{OepnAI} respository uses the Adam optimizer \cite{KB15}. The size of the replay buffer is set to $2000$.
The learning rate used is $10^{-3}$.
The probability $\epsilon$ of taking a random action is initially set to $1$. This parameter reduces over the training steps, and becomes equal to $0.02$ at the end of the training.
The network used is a multi-layer perceptron which, by default, uses two fully connected hidden layers, each with 64 nodes.
Since we are interested in mean-cost objective, the discount factor $\gamma$ is set to $1$.
We observed that reducing the value of $\gamma$ leads to poorer results.
The values reported for both MCTS and deep $Q$-learning are obtained as an average cost over 600 steps.
% which is also the number of steps over which we report the average of the values observed in MCTS.

% We also noted an approximate expected load of every task\todo{D:remove?}, that is the ratio of the expected execution time of the task system over a large time interval $t$ and the time interval $t$.
% As one would expect, we observe that when the expected load of a system is more than $1$, then a nonzero cost is incurred.
% If the load is more than one, then for MGS and EDF on hard tasks, all of the soft tasks cannot be scheduled before their respective deadlines, and we thus miss deadlines of soft tasks incurring the costs.

\paragraph{\bf Conclusions}
% In our model-free deep-RL algorithms, our experimental results have been shown where we train the systems for $10000$ steps.
While deep Q-learning provides good results for small task systems with 3-4 tasks with several thousands of states, this method does not perform well for the benchmarks with large number of tasks.
We trained the task system with 10 soft tasks with deep Q-learning for several million steps, but the state space was found to be too large to learn a good strategy, %even with such large number of training steps. 
and the resulting output produced a cost that is much higher than that observed with MCTS.

Overall, our experimental results show that MCTS consistently provides better results, in particular when the task systems are large, with huge state spaces. This can be explained by the fact that MCTS optimizes locally using information about multiple possible ``futures'' while deep Q-learning rather optimizes globally using information about the uniquely observed trace. 
% Indeed, since the MDPs are huge, this local vs. global approach would explain the differences in performance.
%\todo{Add an informal explanation why this is the case.}
%
We observe that the performance of MCTS with EDF advice is only slightly worse than MCTS with MGS advice.
% Computing the most general safe strategy does not scale well to systems with many hard tasks,
EDF guarantees safety and does not require computing the most general safe strategy, therefore it forms a good heuristic for systems with many hard tasks, where MGS computation becomes too expensive.

% In future work, we consider using Deep-$Q$ learning as a way to evaluate leaf states of the MCTS tree, in order to increase its accuracy.
In future work, we consider using Deep-$Q$ learning in either a selection advice for MCTS
or as a complement to simulations when evaluating new states.

\bibliographystyle{splncs04}
\bibliography{refs}

% \newpage
% \appendix
% \input{appendix}

\end{document}